\documentclass{article}

\usepackage[utf8]{inputenc} 
\usepackage[T1]{fontenc}    
\usepackage{hyperref}       
\usepackage{url}            
\usepackage{booktabs}       
\usepackage{amsfonts}       
\usepackage{nicefrac}       
\usepackage{microtype}      
\usepackage{xcolor}
\usepackage{graphicx}
\usepackage{bbm}

\usepackage{algorithm}
\usepackage{algpseudocode}
\usepackage{amsmath,amssymb}
\usepackage{siunitx}

\usepackage{amsthm}
\newtheorem{proposition}{Proposition}
\newtheorem{corollary}{Corollary}

\usepackage[square,numbers]{natbib}
\usepackage[preprint]{neurips_2026}

\usepackage{booktabs}
\usepackage{xcolor}
\usepackage{colortbl}
\usepackage{pifont}
\usepackage{multirow}
\newcommand{\cmark}{\ding{51}}  
\newcommand{\xmark}{\ding{55}}  

\newcommand{\proc}[1]{\textsc{#1}}

\title{Mechanistic Reaction Prediction via Discrete Flow Matching on Graph-Structured Electron Occupation}

\author{%
  David S.~Hippocampus\thanks{Use footnote for providing further information
    about author (webpage, alternative address)--\emph{not} for acknowledging
    funding agencies.} \\
  Department of Computer Science\\
  Cranberry-Lemon University\\
  Pittsburgh, PA 15213 \\
  \texttt{hippo@cs.cranberry-lemon.edu} \\
}
\author{Xuan-Vu Nguyen\textsuperscript{1}, 
  Octavian Susanu\textsuperscript{1},  Daniel Armstrong\textsuperscript{1}, 
  Philippe Schwaller\textsuperscript{1,2} \\
  \textsuperscript{1}\'Ecole Polytechnique F\'{e}d\'{e}rale de Lausanne (EPFL) \\
  \textsuperscript{2}National Centre of Competence in Research (NCCR) Catalysis \\
  \texttt{\{nguyen.nguyen,philippe.schwaller\}@epfl.ch} \\
}

\begin{document}

\maketitle

\begin{abstract}
    Chemical reactions are fundamentally transformations in electron space, yet most machine learning approaches model them either through \textit{de novo} generation of product molecules or through heuristic graph edits that operate directly on molecular topology.
    We introduce MAELLE (\textbf{M}ech\textbf{A}nistic \textbf{E}dit f\textbf{L}ow-matching on e\textbf{L}ectron r\textbf{E}arrangements), which instead models reactions as discrete flow matching over electron occupation vectors.
    Concretely, we formulate the reactant-to-product mapping as a Continuous-time Markov Chain (CTMC) over the graph-structured integer-valued electron occupation space defined on all bonding, non-bonding, and hydrogen sites.
    To construct the interpolants between the reactants and products, we generalize the discrete flow matching mixture path to an edit-based formulation, where the electron moves are interpolated using Optimal Transport, yielding a mechanism-like set of moves without elementary step annotations.
    MAELLE achieves competitive performance on the USPTO-480K benchmark compared with leading reaction prediction models.
    Beyond in-distribution accuracy, we evaluate robustness across two out-of-distribution settings - structural complexity and reaction type - and find that MAELLE maintains strong performance where existing methods degrade.
    Finally, because the learned flow operates over the full electron redistribution, MAELLE naturally recovers mechanistic trajectories that align with known chemistry and can predict side products of a reaction. 
\end{abstract}


\section{Introduction}

Predicting the outcome of a chemical reaction is one of the central problems at the interface of chemistry and machine learning, with consequences for drug discovery, materials design, and the broader question of how to realize a target molecule once it has been proposed \citep{coley2026grand,corey1969computer}. The space of possible reactions is larger than the already intractable space of drug-like molecules, and a single retrosynthetic plan may chain together a dozen or more steps, each of which must succeed for the route to be viable \citep{schwaller2020predicting}. Reliable prediction, with calibrated uncertainty, is therefore a prerequisite for automating any meaningful portion of synthetic chemistry. 

Reaction prediction has cycled through a succession of representations, each reflecting the dominant ML paradigm of its moment: fixed molecular fingerprints classified by similarity~\citep{morgan1965generation, rogers2010extended, Segler2017neural, coley2017prediction}; SMILES-to-SMILES translation with sequence and Transformer models~\citep{Schwaller2018found, schwaller2019molecular}; hybrid graph-encoder/SMILES-decoder architectures~\citep{coley2019graph, tu2022permutation}; and, most recently, generative graph-edit approaches that produce the product by editing the reactant graph directly~\citep{sacha2021molecule, do2019graph}.

All of these approaches treat reactions as opaque input-to-output mappings, offering no account of the underlying transformation, while chemists reason via arrow-pushing diagrams that track electron flow between donor and acceptor sites~\citep{grossman2003art}. As black-box predictors, they are exposed to spurious shortcuts in USPTO -- patent-level stylistic signatures~\citep{xuan2025tempre} and substrate-level cues (e.g., a single free hydroxyl among protected ones telegraphing an alcohol-selective transformation). Top-$k$ accuracy on standard benchmarks therefore conflates first-principles reactivity reasoning with exploitation of such shortcuts, making out-of-distribution evaluation a more faithful test~\citep{gil2023holistic,bradshaw2025challenging,tran2026quantifying}.

A few recent works have sought to incorporate mechanistic reasoning into reaction prediction. Early work by \citet{Kayala2011} and \citet{kayala2012reactionpredictor} pioneered the prediction of reactions at the mechanistic level by learning electron-flow patterns, decomposing the problem into identifying reactive sites and predicting electron sources and sinks. \citet{bradshaw2018generative} later extended this direction with a GNN-based model that predicts elementary steps for reactions with linear electron flow (LEF) topology, though this restricts coverage to a subset of reaction mechanisms. More recently, \citet{joung2025electron} model electron redistribution via continuous flow matching over bond-electron (BE) matrices, but require mechanistic trajectories imputed from expert-curated templates~\cite{joung2024reproducing,tavakoli2024pmechdb,tavakoli2023rmechdb}, limiting scalability to new reaction types.

In this work, we introduce MAELLE (\textbf{M}ech\textbf{A}nistic \textbf{E}dit f\textbf{L}ow-matching on e\textbf{L}ectron r\textbf{E}arrangements), which models reactions as discrete flow matching (DFM) over integer-valued electron occupation vectors defined on the molecular graph (Figure~\ref{fig:1st_fig}). MAELLE formulates reaction prediction as a Continuous-time Markov Chain (CTMC) in the electron space: given a reactant, the model learns a rate function that drives the electron occupation from the reactant state to the product state through a sequence of sparse, interpretable edit moves, without requiring elementary step annotations. The intermediate trajectory is constructed via optimal transport over electron sites, and the mixture path factorizes over an auxiliary move-state space, enabling tractable training through the discrete flow matching framework.

\begin{figure}[t]
    \centering
    \includegraphics[width=\linewidth]{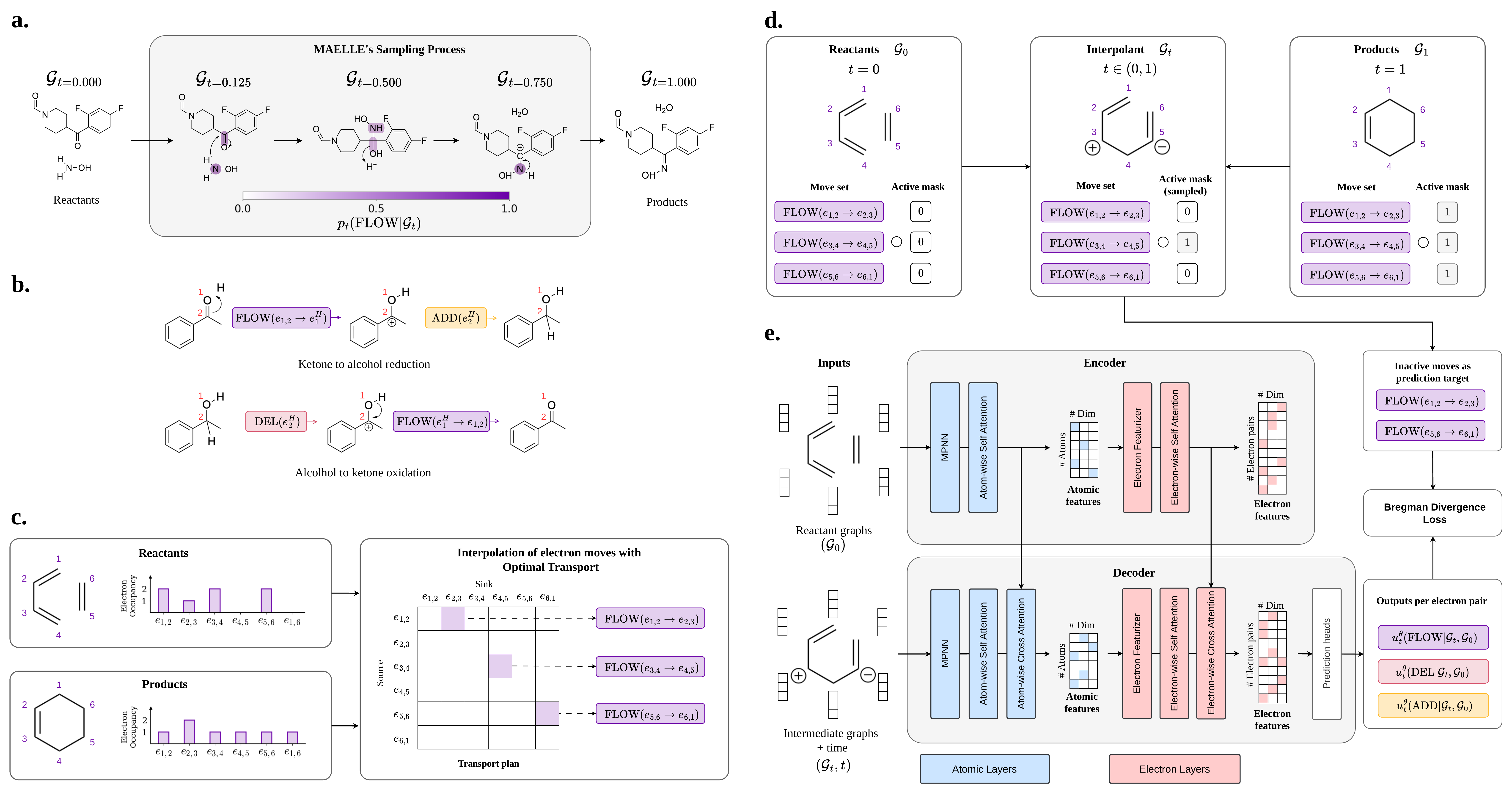}
    \caption{Conceptual overview of MAELLE: \textbf{a.} MAELLE's sampling process with mechanism-like trajectories that model the transition from reactants to products as CTMC.  \textbf{b.} The three elementary actions on electrons, \textsc{FLOW}, \textsc{DEL}, and \textsc{ADD}, demonstrated on an oxidation and a reduction reactions. At training time: \textbf{c.} Electron actions are interpolated using OT; \textbf{d.} interpolant graphs are sampled by sampling a subset of electron moves and applying them to the reactant graphs. \textbf{e.}~MAELLE model architecture and learning objective.}
    \label{fig:1st_fig}
\end{figure}

Our contributions are: (i) we generalize DFM to graph-structured electron spaces with three chemistry-aware edit operations (flow, deletion, addition), and lift the problem to an auxiliary move-state space derived from optimal transport over electron occupations, yielding a factorized mixture path with tractable training and no need for elementary-step labels; (ii) on USPTO-480K we are competitive with the best SMILES- and graph-edit baselines, and on out-of-distribution splits along molecular weight and reaction type we outperform all baselines; (iii) framing reaction prediction generatively gives MAELLE calibrated, ranked product predictions, enabling side-product prediction validated by LLM-as-judge plausibility.

\section{Related Work}
\begin{table}[t]
\caption{Comparison of reaction prediction approaches. 
}
\label{tab:comparison}
\centering
\small
\setlength{\tabcolsep}{5pt}
\begin{tabular}{@{}lcccccc@{}}
\toprule
\textbf{Method} 
& \textbf{Input}
& \textbf{Output} 
& \textbf{Mech. step} 
& \textbf{Directional} 
& \textbf{Broad} 
& \textbf{Atom} \\
 
& \textbf{} 
& \textbf{}
& \textbf{label-free} 
& \textbf{elec. re-dist.} 
& \textbf{coverage} 
& \textbf{conservative}\\
\midrule
Molecular Transformer~\cite{schwaller2019molecular} 
& SMILES & SMILES & \cmark & \xmark & \cmark & \xmark \\
Graph2SMILES~\cite{tu2022permutation}             
& SMILES & SMILES & \cmark & \xmark & \cmark & \xmark \\
Chemformer~\cite{irwin2022chemformer}               
& SMILES & SMILES & \cmark & \xmark & \cmark & \xmark \\
\midrule
MEGAN~\cite{sacha2021molecule}                       
& Graph & Graph edits & \cmark & \xmark & \cmark & \xmark \\
GTPN~\cite{do2019graph}                              
& Graph & Graph edits & \cmark & \xmark & \cmark & \cmark \\
\midrule
NERF \cite{bi2021non}
& Graph & Electron re-dist. & \cmark & \xmark & \cmark & \cmark \\
FlowER~\cite{joung2025electron}                       
& Graph & Electron re-dist. & \xmark$^\dagger$ & \cmark & \xmark$^\dagger$  & \cmark \\
ELECTRO~\cite{bradshaw2018generative}                   
 & Graph & Electron re-dist. & \cmark & \cmark & \xmark$^\ddagger$ & \cmark \\
\rowcolor{blue!6}
MAELLE (ours)                               
& Graph & Electron re-dist. & \cmark & \cmark & \cmark & \cmark$^\S$ \\
\bottomrule
\end{tabular}

\vspace{4pt}
{\footnotesize 
$^\dagger$FlowER covers 86 expert-curated reaction types and requires a mechanistic dataset with imputed elementary steps.\\
$^\ddagger$ELECTRO is restricted to reactions with linear electron flow topology.\\
$^\S$MAELLE is heavy-atom conservative because protonation and deprotonation are simplified.
}
\end{table}

\subsection{Translation-based reaction prediction}
Early work by \citet{schwaller2019molecular} cast reaction prediction as
translation from the SMILES string of the reactants to that of the products.
Because a molecule admits many valid SMILES strings, these models are typically
trained with SMILES augmentation -- multiple strings per molecule
\citep{schwaller2019molecular, irwin2022chemformer} -- which is data-inefficient.
Graph2SMILES \citep{tu2022permutation} partially addresses this by encoding the
reactants as a 2D graph, removing the need to augment the input, while still
decoding the product as a SMILES string with a Transformer decoder.

Sequence representations such as SMILES \citep{daylightsmiles} and SELFIES
\citep{krenn2020self} make products easy to generate with well-established
autoregressive text models, but they come at a cost: the model generates the
product \emph{de novo}, with no structural link to the reactants. This has three
consequences. First, generation is unconstrained, so there is no guarantee that
the predicted reaction is atom-conservative -- a critical chemical constraint.
Second, because the reactant-product correspondence is not explicit, recovering
what actually changed requires a separate atom-mapping tool
\citep{schwaller2021rxnmapper}; and even when a prediction appears
atom-conserving, nothing ensures a valid mechanistic pathway connects reactants
to products, so a chemically sound transformation and a hallucinated one can
look alike. Third, generating the entire product from scratch makes these models
degrade on molecules whose complexity falls outside the training domain, even
when the underlying chemistry is simple \citep{gil2023holistic,bradshaw2025challenging,tran2026quantifying}.

\subsection{Reactions as graph edits}

To build the chemical change directly into the model, another line of work
predicts reactions as topological transformations of the reactant graph, and
atom conservation can be enforced through the choice of edit operations. MEGAN
\citep{sacha2021molecule} uses EditAtom, EditBond, AddAtom, and AddBenzene,
explicitly modeling a reaction as a sequence of graph edits; operations such as
EditAtom, AddAtom, and AddBenzene break atom conservation, though one can
restrict the operation set to bond-only edits to preserve it, as in GTPN
\citep{do2019graph}.

Graph edits make the transformation explicit and give atom mapping for free, and
atom conservation is achievable -- but the edits are purely topological, so no
mechanistic pathway is guaranteed to connect the two endpoints. Bond formation
is the clearest case: a graph-edit model simply inserts a bond between two atoms,
whereas chemically the electrons forming that bond are drawn from elsewhere in
the molecule.



\subsection{Reactions as electron redistribution}

At a level deeper than topological graph edits, a reaction is a redistribution of
electrons: a sequence of elementary steps in which electrons flow from a source
to a sink, described in chemistry by arrow-pushing notation. Modeling reactions
this way is informative about what actually happens, but capturing the
elementary steps explicitly has traditionally required expert-labeled mechanisms
or heavy physical calculation. \citet{Kayala2011} assemble a dataset of 1{,}630
reactions with 2{,}989 mechanistic steps and train a model to rank orbital
reactivity for a single elementary step, wrapping it in a recursive loop to
reach the product. In the same spirit, \citet{joung2025electron} scale the
mechanistic dataset to 290K reactions and over 2 million elementary steps, and
train a flow-matching model over the bond-electron matrix. Both are informative
but depend on expert-labeled elementary steps, and are therefore less scalable
than the end-to-end approaches above.

ELECTRO \citep{bradshaw2018generative} removes the labeling requirement,
interpolating electron moves directly from end-to-end reactant-product pairs.
However, it is restricted to reactions with linear electron flow topology --
those whose mechanism has a well-defined start and end to the electron flow --
and so does not cover cyclic mechanisms.
NERF \citep{bi2021non} also avoids
mechanism labels but takes a non-autoregressive, one-shot view: it predicts, for
each bond, the net number of electron pairs added or removed between reactants
and products. Because it models only the change at each bond and not a transfer
between sites, the direction of electron flow is left implicit -- the prediction
records that a bond gained or lost electrons, but not where those electrons came
from or went. It also models bond electrons only, excluding lone pairs and other
valence electrons from the representation.

MAELLE models reactions as directional electron redistribution while guaranteeing
heavy-atom conservation (Table~\ref{tab:comparison}). Like ELECTRO, it needs no
elementary-step labels, but it is not limited to linear flow: our FLOW operation
moves an electron pair from a source site to a sink site, making the direction of
flow explicit -- answering where electrons come from and where they go -- and
composing freely into cyclic as well as linear mechanisms. In place of mechanism
labels, we recover the net electron moves with Optimal Transport
(Section~\ref{sec:ot}), sample interpolant graphs between reactants and products,
and train a discrete flow-matching model over these electron edits.

\subsection{Discrete generative modeling on graphs}

From the standpoint of generative modeling, discrete generative models on graphs typically cast generation as iterative substitution of node and edge types under a denoising diffusion or flow-matching process that maps from a source distribution that is easy to sample, such as a uniform distribution, to the data distribution. 
DiGress \cite{vignac2022digress} and DeFoG \cite{qin2024defog} generate graphs by employing discrete denoising diffusion processes \cite{sohl2015deep,ho2020denoising} and discrete flow matching \cite{campbell2024generative}, respectively.

More related to our setting,
RetroBridge \cite{igashov2024retrobridge} tackles the problem of retrosynthesis, i.e., predicting possible precursors given the products, in a reversed setting compared to our work. 
The model learns a bridge between two coupled distributions of products and reactants via Markov bridge \cite{ccetin2016markov}.

Both lines model generation as resampling of categorical node/edge types — from a noise prior (DiGress, DeFoG) or a coupled endpoint distribution (RetroBridge). MAELLE instead operates over integer-valued electron occupation, where transitions are electron transfers (flow/add/delete) rather than type substitutions; this makes the process atom-conserving by construction and yields mechanistically interpretable trajectories.

\section{Background}
\subsection{Discrete Flow Matching}
\label{sec:dfm-bg}
Let $\mathcal{T}$ be a vocabulary and $\mathcal{X}=\mathcal{T}^{*}=\bigcup_{n=0}^{N}\mathcal{T}^{n}$
the space of token sequences of length at most $N$. Discrete flow matching
(DFM) \citep{campbell2024generative, gat2024discrete} learns a Continuous-time
Markov Chain (CTMC) that transports a source distribution $p_0$ at $t=0$ to a
target $p_1$ at $t=1$. As in the continuous case, the central construction is a
\emph{conditional mixture path} $p_t(x\mid x_0,x_1)$ with boundary conditions
$p_0(x\mid x_0,x_1)=\delta_{x_0}(x)$ and $p_1(x\mid x_0,x_1)=\delta_{x_1}(x)$;
the endpoints $(x_0,x_1)$ are drawn from a coupling $\pi(x_0,x_1)$ whose
marginals are the source and target, and marginalizing the conditional path over
$\pi$ gives the marginal path $p_t$ that interpolates $p_0$ and $p_1$.

For sequences of a common length $n$, the standard choice is the factorized,
token-wise mixture path
\begin{equation}
  p_t(x^i \mid x_0, x_1)
  = (1-\kappa(t))\,\delta_{x_0^i}(x^i) + \kappa(t)\,\delta_{x_1^i}(x^i),
  \qquad
  p_t(x \mid x_0, x_1) = \prod_{i=1}^{n} p_t(x^i \mid x_0, x_1),
  \label{eq:tokenwise}
\end{equation}
where the schedule $\kappa:[0,1]\to[0,1]$ is monotone with $\kappa(0)=0$,
$\kappa(1)=1$. Each token independently switches from its source value $x_0^i$
to its target value $x_1^i$, and does so by time $t$ with probability
$\kappa(t)$. The CTMC that generates \eqref{eq:tokenwise} has the token-wise
marginal rate
\begin{equation}
  u_t(x^i \mid x_t)
  = \frac{\kappa'(t)}{1-\kappa(t)}\,
    \bigl[\,p_{1\mid t}(x^i \mid x_t) - \delta_{x_t^i}(x^i)\,\bigr],
  \label{eq:tokenwise-rate}
\end{equation}
where $p_{1\mid t}(x^i\mid x_t)$ is the model's posterior over the target token
at position $i$. Training therefore reduces to predicting $p_{1\mid t}$ per
position, and the coefficient $\kappa'(t)/(1-\kappa(t))$ is the rate at which a
position that has not yet reached its target switches to it.

\subsection{Edit Flows}
\label{sec:editflow-bg}
The token-wise path \eqref{eq:tokenwise} assumes a fixed, common length and a position-to-position correspondence between $x_0$ and $x_1$. Edit Flows
\citep{havasi2025edit} remove this assumption, allowing sequences of differing length related by \emph{edit operations} -- insertion, deletion, and substitution -- rather than position-wise substitution alone. To recover a factorized path in this setting, the construction is lifted to an auxiliary space of aligned (padded) sequences $\mathcal{Z}=(\mathcal{T}\cup\{\epsilon\})^{N}$,
where $\epsilon$ is a padding symbol, together with the map
$f_{\mathrm{rm-blanks}}:\mathcal{Z}\to\mathcal{X}$ that deletes all
padding. Given endpoints $x_0,x_1$ and a chosen alignment $z_0,z_1\in\mathcal{Z}$
with $f_{\mathrm{rm-blanks}}(z_j)=x_j$, the mixture path is defined on
the augmented space $\mathcal{X}\times\mathcal{Z}$ as
\begin{equation}
  p_t(x, z \mid z_0, z_1)
  = p_t(z \mid z_0, z_1)\,\delta_{f_{\mathrm{rm}\text{-}\epsilon}(z)}(x),
  \label{eq:editflow-aug}
\end{equation}
where $p_t(z\mid z_0,z_1)$ is the token-wise path \eqref{eq:tokenwise} applied on
the aligned sequences. The data-space path is the pushforward of the auxiliary
path through $f_{\mathrm{rm-blanks}}$: the factorization lives in
$\mathcal{Z}$, while the observed sequence $x=f_{\mathrm{rm-blanks}}(z)$
can have variable length. Because each aligned position independently takes its
source or target value, this is equivalent to sampling edit operations to
apply to $x_0$.
We adopt this view in this paper, and show that it is equivalent to the token-wise sampling in Section~\ref{sec:dfm} (Proposition~\ref{prop:equiv}).

\section{Pseudo mechanistic reaction prediction with MAELLE}

\subsection{Electron occupation as the main source of change}
\label{sec:problem_statement}
We represent a molecule as an augmented graph $\mathcal{G} = (\mathcal{V}, \mathcal{A}, \mathcal{E}, o)$, where $\mathcal{V}$ is the set of heavy atoms, $\mathcal{A}$ is the heavy-atom adjacency, and $\mathcal{E}$ is a fixed set of \emph{electron sites} comprising bonding sites $e_{ij}$ ($i \neq j$), lone-pair sites $e_{ii}$, and per-atom hydrogen sites $e^H_i$ (hydrogens on the same heavy atom are topologically equivalent and pooled into one site~\cite{chen2025predicting}). 
The occupation vector $o \in \mathbb{Z}^{|\mathcal{E}|}$ counts electron pairs at each site, including virtual ones with $o^e = 0$.
See Appendix~\ref{app:problem_statement} for the formal definitions and an illustration on ethanol.

A chemical reaction maps reactants $\mathcal{G}_x$ to products $\mathcal{G}_y$ by redistributing electrons, i.e., changing $o$. The atom set $\mathcal{V}$ is preserved. Moreover, the heavy-atom adjacency can be recovered deterministically from the occupation through
\begin{equation}
    a_{ij} = \mathbbm{1}[\,o^{e_{ij}} > 0\,],
    \label{eq:o_to_a}
\end{equation}
so $p(\mathcal{G}_y \mid \mathcal{G}_x) = p(o_y \mid \mathcal{G}_x)$. 
Thus, we could model a reaction by acting entirely on the vector $o$.

\subsection{Edit operations on electron occupation}

In this section, we describe the electron edit operations that enacts on $o$.
Let $\mathbf{1}_{e} \in \{0,1\}^{|\mathcal{E}|}$ denote the one-hot vector with a single nonzero entry at the position corresponding to site $e \in \mathcal{E}$.
We define the following actions on $o$ (Figure \ref{fig:1st_fig}b):

\paragraph{FLOW} Transfer one electron pair from site $e_{\text{source}}$ to site $e_{\text{sink}}$, requiring $o(e_{\text{source}}) > 0$:
    \begin{equation}
        \label{eq:flow}
        \textsc{flow}(e_{\text{source}} \rightarrow e_{\text{sink}})(o) = o - \mathbf{1}_{e_{\text{source}}} + \mathbf{1}_{e_{\text{sink}}}
    \end{equation}

\paragraph{DEL} Remove one electron pair from site $e$, requiring $o^e > 0$:
    \begin{equation}
        \label{eq:del}
        \textsc{del}(e)(o) = o - \mathbf{1}_{e}
    \end{equation}

\paragraph{ADD} Add one electron pair to site $e$:
    \begin{equation}
        \label{eq:add}
        \textsc{add}(e)(o) = o + \mathbf{1}_{e}
    \end{equation}

The sparsity of these three operations enables us to define a tractable distribution of the intermediate states between the occupation vectors of the reactants and products $o_x$ and $o_y$.

\subsection{Interpolation of pseudo mechanistic steps via Optimal Transport}
\label{sec:ot}
To obtain a set of moves that transform $o_x$ into $o_y$, we solve a balanced integer Optimal Transport problem over the electron sites. To handle reactions where the total electron count is not conserved, we augment the site set with a virtual sink/source $e_\varnothing$, yielding $\bar{\mathcal{E}} = \mathcal{E} \cup \{e_\varnothing\}$ with augmented occupations $\bar{o}_x, \bar{o}_y$. The transport cost between two sites is the minimum shortest-path distance between their constituent atoms on either the reactant or product graph, with a large penalty $c_\varnothing$ for transport to or from $e_\varnothing$.
The optimal transport plan could be solved with the simplex algorithm, where the design of the objective, marginal constraints, and cost matrix can be found in  Appendix~\ref{app:ot_details}.
This results in the optimal transport matrix $T^* \in \mathbb{Z}^{|\bar{\mathcal{E}}| \times |\bar{\mathcal{E}} |}$.

\paragraph{The move set as a deterministic alignment.} Each non-zero off-diagonal entry $T^*_{e,e'} > 0$ is translated to an electron action $m \in \{\textsc{Flow}, \textsc{Del}, \textsc{Add}\}$ (Eqs.~\ref{eq:flow}--\ref{eq:add}):

\begin{equation}
    m = 
    \begin{cases}
        \textsc{FLOW}(e \rightarrow e') \quad \text{if} \quad e \neq e_\varnothing, e' \neq e_\varnothing \\
        \textsc{DEL}(e) \quad \text{if} \quad  e \neq e_\varnothing, e' = e_\varnothing \\
        \textsc{ADD}(e') \quad \text{if} \quad e = e_\varnothing, e' \neq e_\varnothing \\
    \end{cases}
    \label{eq:source_sink}
\end{equation}

Together, they form an un-ordered set of move $\mathcal{M}$, such that we would recover $o_y$ if we fully apply $\mathcal{M}$ to $o_x$:

\begin{equation}
  \mathcal{M} \;=\; \{m_1,\dots,m_K\},
  \qquad K \;=\; \sum_{e\neq e'} T^\ast_{e,e'},
  \label{eq:moveset}
\end{equation}
\begin{equation}
    \mathrm{ApplyEdit}(o_x, \mathcal{M}) = o_y
\end{equation}

\subsection{Discrete Flow Matching over Electron Edits}
\label{sec:dfm}

\paragraph{Edit-based mixture path}
Let $\mathcal{S} \subseteq \mathcal{M}$ be a subset of the moves, interpreted as
the moves that have been applied so far. We define the interpolant occupation corresponding to $\mathcal{S}$ as $o_{\mathcal{S}} = \mathrm{ApplyEdit}(o_x, \mathcal{S})$, which by
\eqref{eq:moveset} recovers $o_x$ at $\mathcal{S}=\varnothing$ and $o_y$ at
$\mathcal{S}=\mathcal{M}$. To interpolate between the two endpoints, we place a
distribution over which subset has been applied at time $t$ and read the
occupation off through $\mathrm{ApplyEdit}$. Following a monotone schedule
$\kappa:[0,1]\to[0,1]$ with $\kappa(0)=0$, $\kappa(1)=1$, each move is applied
independently with probability $\kappa(t)$, giving the edit-based mixture path
\begin{equation}
  p_t(o \mid o_x, o_y)
  =
  \sum_{\mathcal{S} \subseteq \mathcal{M}}
  \underbrace{
  \kappa(t)^{|\mathcal{S}|}\,\bigl(1-\kappa(t)\bigr)^{|\mathcal{M}|-|\mathcal{S}|}
  }
  _{\displaystyle p_t(\mathcal{S}\,\mid\, o_x,o_y)}
  \;\mathbbm{1}\!\left[\,\mathrm{ApplyEdit}(o_x, \mathcal{S}) = o\,\right]
  \label{eq:edit-mixture}
\end{equation}
Equivalently, the path is the law of $\mathrm{ApplyEdit}(o_x, \mathcal{S}_t)$
where $\mathcal{S}_t$ is obtained by including each move $m_k \in \mathcal{M}$
independently with probability $\kappa(t)$: sampling an interpolant amounts to flipping one $\mathrm{Bernoulli}(\kappa(t))$ mask per move and applying the selected subset to the reactant occupation (Algorithm~\ref{alg:sample}). Since
each move is an independent binary variable, the weight
$p_t(\mathcal{S}\mid o_x,o_y)$ factorizes over $\mathcal{M}$, and the endpoints
follow directly: at $t=0$ no moves are sampled so $p_0=\delta_{o_x}$, and at $t=1$
all moves are applied so $p_1=\delta_{o_y}$.

The path \eqref{eq:edit-mixture} is the electron-space instance of the
augmented-space mixture path of Edit Flows~\citep{havasi2025edit}. There,
the path is defined over an auxiliary space of padded sequences and pushed to
token sequences by removing padding; here it is defined over subsets of
electron moves and pushed to occupations by $\mathrm{ApplyEdit}$. In both, an
interpolant is obtained by sampling which edits to apply, and the two constructions yield the same conditional path. We state this equivalence in Proposition~\ref{prop:equiv}.

\paragraph{Conditional and marginal rates}
The mixture path \eqref{eq:edit-mixture} is generated by a CTMC whose
conditional rate fires each unapplied move independently. Following the standard
DFM derivation, the rate of switching a move $m_k$ from unapplied to applied is
the hazard rate of its $\mathrm{Bernoulli}(\kappa(t))$ flip,
\begin{equation}
  u_t\bigl(m_k \mid o_t, o_x, o_y\bigr) \;=\; \frac{\kappa'(t)}{1-\kappa(t)},
  \label{eq:cond-rate}
\end{equation}
with the reverse rate zero (moves fire once and stay fired). This conditional
rate assumes the product $o_y$ is known. At inference it is not, so the training
target is the \emph{marginal} rate, obtained by taking the expectation over the
posterior of the endpoints given the current state,
\begin{equation}
  u_t(m \mid o_t, \mathcal{G}_x)
  \;=\;
  \mathbb{E}_{p_t(o_y \mid o_t, \mathcal{G}_x)}\;
  u_t\bigl(m \mid o_t, o_x, o_y\bigr),
  \label{eq:marginal-rate}
\end{equation}
which depends only on the current state and the reactant, and is therefore
learnable by a neural network. By the auxiliary-variable result of
\citet{havasi2025edit}, a model trained to match the conditional rate in the move space yields the correct marginal rate after marginalization.

\paragraph{Training}
We model the reactant-to-product mapping as a coupling $\pi(\mathcal{G}_x, \mathcal{G}_y)$ over reactant-product pairs. 
At training time, $\pi$ is the empirical distribution of reactions in the training set, each of which provides one coupled pair.
The training target is the marginal rate $u_t(m\mid o_t, \mathcal{G}_x)$ of
Eq.~\eqref{eq:marginal-rate}, which we approximate with a graph neural network
$u^\theta_t(m\mid \mathcal{G}_x, \mathcal{G}_t, t)$.
We adopt an encoder--decoder architecture: the encoder processes the fixed reactant $\mathcal{G}_x$ into conditioning representations, and the decoder processes the interpolant $\mathcal{G}_t$ and cross-attends to them at every layer, so the predicted rates are conditioned on the reactant throughout, analogous to conditioning on the source in conditional flow matching. 
Unlike standard molecular GNNs, which typically featurize
only atoms and bonds, our edits act on the electron occupation $o$ over the full site set $\mathcal{E}$, including bonding sites, lone pairs, and per-atom hydrogen
sites, so an \textsc{ElectronFeaturizer} lifts atom representations into a
vector for every electron site (Figure~\ref{fig:1st_fig}e). 
The per-move rates are read off
these electron-site representations, matching the granularity at which the edits operate. Full architectural details can be found in Appendix~\ref{app:architecture}.

We fit $u^\theta_t$ with the Bregman-divergence objective of
\citet{holderrieth2024generator, havasi2025edit}, which for a CTMC rate
reduces to a total-rate penalty minus a log-rate reward on the correct
transitions. Let $\mathcal{M}_{\mathrm{target}}=\mathcal{M}\setminus\mathcal{S}$ be
the moves not yet applied at the sampled interpolant. 
The loss is

\begin{equation}
  \mathcal{L}(\theta)
  \;=\;
  \mathop{\mathbb{E}}_{
  \substack{
  (\mathcal{G}_x,\mathcal{G}_y)\sim\pi\\
        t\sim\mathcal{U}(0,1)\\
        \mathcal{S}\sim p_t(\cdot\mid o_x,o_y)}
        }
  \left[\;
    \underbrace{\sum_{m}u^\theta_t(m\mid\mathcal{G}_x, \mathcal{G}_t, t)}_{\displaystyle \Lambda_{\theta,t}}
    \;-\;
    \frac{\kappa'(t)}{1-\kappa(t)}
    \sum_{m\in\mathcal{M}_{\mathrm{target}}}
      \log u^\theta_t(m\mid\mathcal{G}_x, \mathcal{G}_t, t)
  \;\right]
  \label{eq:loss}
\end{equation}

Because a \textsc{FLOW} move acts on pairs of atoms, and the \textsc{ADD} move could act on virtual electron sites, calculating the rates of all possible moves would make the term $\Lambda_{\theta,t}$ scale as $O(|\mathcal{E}|^2)$. We therefore factorize the move's
rate for these two operations into a per-site \emph{source} rate and a \emph{sink} distribution over
destinations:

\begin{equation}
  u^\theta_t(m \mid \cdot) =
  \begin{cases}
    \lambda^{\mathrm{flow}}_{\theta,t}(e_{\mathrm{source}} \mid \cdot)\,
      p^{\mathrm{sink,flow}}_\theta(e_{\mathrm{sink}} \mid e_{\mathrm{source}}, \cdot)
      & m = \textsc{flow}(e_{\mathrm{source}}\!\to\!e_{\mathrm{sink}}), \\[0.4em]
    \lambda^{\mathrm{del}}_{\theta,t}(e_{\mathrm{source}} \mid \cdot)
      & m = \textsc{del}(e_{\mathrm{source}}), \\[0.4em]
    \lambda^{\mathrm{add}}_{\theta,t}(v \mid \cdot)\,
      p^{\mathrm{sink,add}}_\theta(e_{\mathrm{sink}} \mid v, \cdot)
      & m = \textsc{add}(e_{\mathrm{sink}}),
  \end{cases}
  \label{eq:rate-param}
\end{equation}

Here $\lambda^{\mathrm{flow}}_{\theta,t}$ and $\lambda^{\mathrm{del}}_{\theta,t}$
are per-site firing rates over non-virtual electron sites $e_{\mathrm{source}}\in\mathcal{E}$, $o^{e_{\mathrm{source}}}>0$;
$\lambda^{\mathrm{add}}_{\theta,t}$ is a per-atom rate over atoms
$v\in\mathcal{V}$, and $p^{\mathrm{sink,flow}}_\theta$, $p^{\mathrm{sink,add}}_\theta$
are separately parameterized distributions over destination sites. For
\textsc{del} no sink is needed, since the operation only removes an electron pair.

We parameterize \textsc{add} through an atom $v$ rather than a site because a new
electron pair is not associated with any existing site a priori. Intuitively,
the model first decides to donate a pair of electrons to atom $v$, and $v$ then
decides where to place them -- selecting a destination site $e_{\mathrm{sink}}$
among its associated sites through a dedicated sink head. In this sense
\textsc{add} behaves as a pseudo-\textsc{flow} whose source is the atom itself
rather than an occupied electron site, though the two use separate sink
parameterizations.
The total rate then factorizes as $\sum_m u^\theta_t(m\mid\cdot) = \sum_{e}(\lambda^{\mathrm{flow}} +
\lambda^{\mathrm{del}}) + \sum_a \lambda^{\mathrm{add}}$, since each
$p^{\mathrm{sink}}$ sums to one.

The training loop can be found in Algorithm~\ref{alg:train}.

\paragraph{Inference}
At inference, the products are generated by the marginal CTMC starting
from the reactants. We initialize $\mathcal{G}_t = \mathcal{G}_x$ at $t=0$ and
take $N$ Euler steps of size $\Delta t = 1/N$. Since the reactant is fixed, the
encoder is run once on $\mathcal{G}_x$ and its representations are reused at
every step, while the decoder predicts the move rates on the current graph
$\mathcal{G}_t$. At each step, every admissible move fires independently as a
thinned Poisson process: a move $m$ with predicted rate
$u^\theta_t(m\mid\mathcal{G}_x,\mathcal{G}_t,t)$ is applied with probability
\begin{equation}
  p_{\mathrm{fire}}(m)
  \;=\;
  1 - \exp\bigl(-\,u^\theta_t(m\mid\mathcal{G}_x,\mathcal{G}_t,t)\,\Delta t\bigr)
  \label{eq:fire}
\end{equation}
The sampled edit moves are applied to $\mathcal{G}_t$ via
$\mathrm{ApplyEdit}$. Iterating to $t=1$ yields a predicted product
$\hat{\mathcal{G}}_1$. The full procedure is given in Algorithms~\ref{alg:inference}.

\paragraph{Confidence and ranking through multiple trajectories}
A chemical reaction may proceed along several competing pathways, yielding a
mixture of products rather than a single deterministic outcome. Because firing
in \eqref{eq:fire} is stochastic, running multiple independent CTMC trajectories from
the same reactant $\mathcal{G}_x$ produces a distribution over products. We
canonicalize the resulting product graphs, group them by identity, and define
the \emph{confidence} of each unique product as its frequency among the
samples; ranking by descending confidence yields a top-$k$ list. 
This ranking plus the edits only acting on the electrons let MAELLE model a distribution of heavy atom-conservative products, capturing not only the main products as high-confidence predictions, but also plausible side products (Section~\ref{sec:side_products}).

\section{Results}

\subsection{Experiments}
\label{sec:experiments}
\paragraph{Datasets and splits.}
All experiments use the USPTO-480K reaction corpus of $\sim$480K patent
reactions~\citep{lowe2012extraction, Jin2017predicting}, evaluated under two
partitionings. The first is the corpus's standard train/val/test split, a common
benchmark for reaction
prediction~\citep{Jin2017predicting, schwaller2019molecular, sacha2021molecule, bi2021non, do2019graph}
(Section~\ref{sec:uspto_480K}). The standard split, however, measures only
in-distribution learning and, as recent work notes, misses the failure modes
that matter most in chemistry~\citep{gil2023holistic, bradshaw2025challenging, tran2026quantifying}.
We therefore construct a second partitioning of the same corpus into three test
sets probing different axes of generalization (Section~\ref{sec:ood}):
\textbf{(i)}~an in-distribution control (IID); \textbf{(ii)}~\emph{OOD Mass},
holding out high-molecular-weight products to probe structural complexity; and
\textbf{(iii)}~\emph{OOD Ester}, holding out non-methyl esterification reactions
to probe generalization to unseen reaction types. Split construction, sizes, and
preprocessing (atom-mapping with RXNMapper~\citep{schwaller2021rxnmapper} and
electron-move interpolation, Section~\ref{sec:ot}) are detailed in
Appendix~\ref{app:data}.

\paragraph{Baselines.}
We compare against representative methods spanning the three prediction
modalities, which lets us probe the inductive biases of each modality rather than
raw accuracy alone: SMILES generation -- Molecular Transformer (MT)~\citep{schwaller2019molecular}
and Graph2SMILES (G2S)~\citep{tu2022permutation}; graph edits --
MEGAN~\citep{sacha2021molecule} and GTPN~\citep{do2019graph}; and electron
redistribution -- NERF~\citep{bi2021non}. We exclude ELECTRO~\citep{bradshaw2018generative}
and FlowER~\citep{joung2025electron}, which by design cannot be trained on
USPTO-480K: ELECTRO applies only to reactions with linear-electron-flow topology,
and FlowER requires elementary-step annotations, unavailable for most USPTO-480K
reactions.

\paragraph{Metrics.}
We report top-$k$ exact-match accuracy, with predicted and reference products
canonicalized before comparison. Beyond exact match -- which penalizes
chemically plausible side products -- we additionally report a \emph{plausibility
rate} judged by a large language model (Section~\ref{sec:side_products};
protocol in Appendix~\ref{app:llm_judge}). The directional-flow evaluation of
Section~\ref{sec:directional} uses a separate reference dataset and its own
coverage metrics, which are discussed there.

\subsection{USPTO-480K benchmark}
\label{sec:uspto_480K}


\begin{table}[t]
\caption{Top-$k$ accuracy comparison on USPTO-480K. Entries taken from the original papers are marked with "$^\dagger$"}
\label{tab:results}
\centering
\small
\setlength{\tabcolsep}{5pt}
\begin{tabular}{@{}llcccc@{}}
\toprule
\textbf{Method} 
& \textbf{Prediction} 
& \textbf{Top-1} 
& \textbf{Top-3} 
& \textbf{Top-5} 
& \textbf{Top-10} \\
& \textbf{modality} & & & & \\
\midrule
Molecular Transformer$^\dagger$~\cite{schwaller2019molecular} 
& SMILES & 0.886 & 0.935 & 0.942 & --- \\
Graph2SMILES$^\dagger$~\cite{tu2022permutation} 
& SMILES & 0.903 & 0.940 & 0.948 & 0.953 \\
\midrule
MEGAN$^\dagger$ \cite{sacha2021molecule} 
& Graph edits & 0.863 & 0.924 & 0.940 & 0.954 \\
GTPN$^\dagger$ \cite{do2019graph}
& Graph edits & 0.832 & 0.860 & 0.865 & --- \\
\midrule
NERF$^\dagger$~\cite{bi2021non} 
& Electron edits & 0.907 & 0.933 & 0.937 & --- \\
\rowcolor{blue!6}
\textbf{MAELLE (ours)} 
& \textbf{Electron flow} & 0.872 & 0.930 & 0.939 & 0.946 \\
\bottomrule
\end{tabular}
\end{table}

Table~\ref{tab:results} compares MAELLE against representative baselines
spanning the three prediction modalities: SMILES generation (Molecular
Transformer~\citep{schwaller2019molecular}, Graph2SMILES~\citep{tu2022permutation}),
graph edits (MEGAN~\citep{sacha2021molecule}, GTPN~\citep{do2019graph}), and
electron redistribution (NERF~\citep{chen2025predicting}). On the standard
USPTO-480K split, MAELLE reaches a top-1 accuracy of 87.2\%, rising to 93.0\%,
93.9\%, and 94.6\% at top-3, top-5, and top-10 respectively. This places it on
par with the graph-edit models -- comparable to MEGAN (86.3\% top-1) and above
GTPN (83.2\%) -- and within 3.5 points of the strongest SMILES-based model,
Graph2SMILES (90.3\%), while trailing the electron-redistribution baseline NERF
(90.7\%) at top-1.
At top-5, MAELLE performs on par with NERF and MEGAN, and
lags only slightly behind the SMILES-based models.

The standard USPTO-480K split measures in-distribution learning, and by this
measure all models capture the training distribution well: at top-5, every
method except GTPN exceeds 93\% accuracy.
We therefore report USPTO-480K
performance to establish that MAELLE is able to learn the coupled distribution between
reactants and products, using the existing approaches as references.
In the next section, we turn to
out-of-distribution evaluation (Section~\ref{sec:ood}), where differences between
modalities are more pronounced.


\subsection{Out-of-distribution benchmark}
\label{sec:ood}

A major pitfall of reaction-prediction models is poor generalization outside their training domain~\cite{gil2023holistic,bradshaw2025challenging,tran2026quantifying}. We re-split USPTO-480K along two axes -- molecular complexity (\emph{OOD Mass}) and reaction novelty (\emph{OOD Ester}, with all non-methyl esterifications held out from training) -- alongside an in-distribution control (IID); see Appendix~\ref{app:ood}.

Top-k accuracies on these three test sets are shown for MAELLE, Graph2SMILES, MolecularTransformers, and MEGAN in Figure \ref{fig:ood}a.
As expected, all models perform on par with each other on the IID test, with top-10 accuracies of approximately 90\%, agreeing with the trend in the test set of the original USPTO-480K benchmark.
On the "OOD Mass", models that predict transformations - MAELLE and MEGAN - perform significantly better than the \textit{de-novo} generative models like Molecular Transformers and Graph2SMILES.
This is attributed to the former not having to reconstruct the products from scratch \cite{xuan2025tempre}.
Indeed, plotting the top-1 accuracy against the molecular weight reveals Molecular Transformers performance decays rapidly as the target becomes more complex, while Graph2SMILES shows a less severe drop. 
The better accuracy of Graph2SMILES compared with Molecular Transformers could be attributed to its graph encoder, which supports the permutation invariance.
On the other hand, MAELLE and MEGAN maintain a stable accuracy across the bins.

While MAELLE and MEGAN perform similarly on OOD Mass, MAELLE shows a significant advantage over MEGAN on the OOD Ester, being the most performant model in this domain.
This highlights the benefit of CTMC framework: Even though the model does not see the exact reaction type in the training data, it most likely has seen the intermediate graph states of reactions similar to esterification, such as amide coupling.

Two examples are shown in Figure \ref{fig:ood}b and \ref{fig:ood}c for the "OOD Mass" and "OOD Ester", respectively.
The first case features a complex macrocycle structure but a very simple amide coupling reaction.
Both MAELLE and MEGAN successfully recover the ground truth, while Graph2SMILES and Molecular Transformers struggle to reconstruct the molecular structure in SMILES format.
The second case involves an esterification reaction with an anhydride.
Out of the four models, only MAELLE correctly predicts the ester as the product

\begin{figure}[t]
    \centering
    \includegraphics[width=0.83\linewidth]{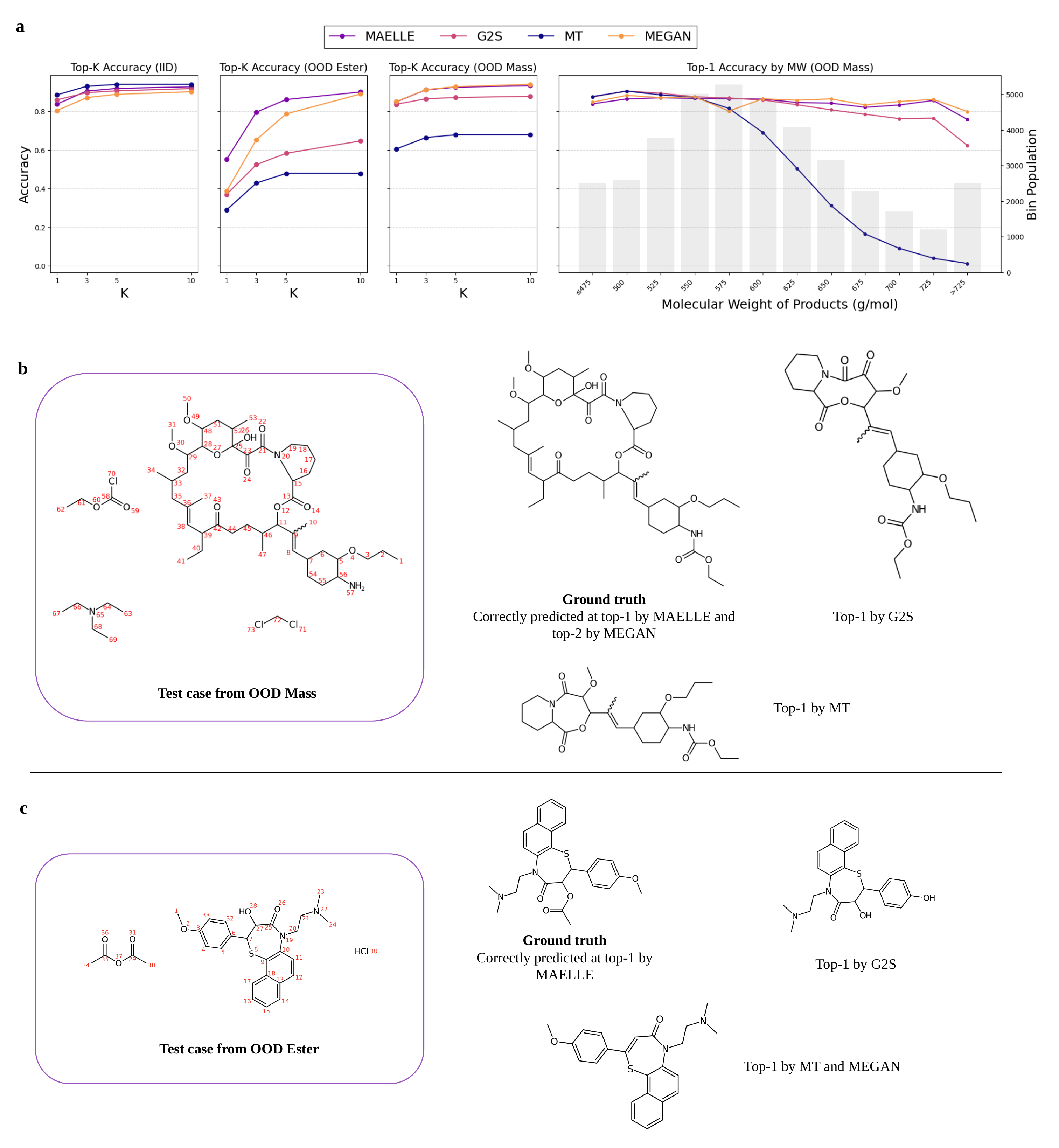}
    \caption{Comparative performance of MAELLE, Graph2SMILES, Molecular Transformers, and MEGAN on our OOD benchmark. \textbf{a. }From left to right, the first three plots show top-k accuracy on in-distribution test set (IID), esterification-split test set (OOD Ester), and molecular-weight-split test set (OOD Mass). The last plot describes the top-1 accuracy divided into weight bins. \textbf{b. } An illustrative example from the OOD Mass split. \textbf{c. }An illustrative example from the OOD Ester split.}
    \label{fig:ood}
\end{figure}

\subsection{Side product prediction with confidence}
\label{sec:side_products}
By modeling reaction as a generative problem, simulating multiple trajectories, and taking the most popular outputs, we are somewhat mimicking how a real chemistry process happens, where multiple pathways compete with others.
To characterize this behavior, we sampled 5000 cases from the test set of USPTO-480K where MAELLE has at least 2 high-confidence predictions with confidence no less than 7.8\%, corresponding with 5 out of 64 samples.
These examples are then matched with the predictions of Molecular Transformers and Graph2SMILES. 
As sequence-based methods, these model relies on beam-search for sampling different outcomes in the token space, in contrast with MAELLE which samples in a more chemistry-friendly electron space.

As recent LLMs have shown remarkable symbolic understanding and chemistry knowledge \cite{bran2025chemical,voinarovska2026humans}, we opted to use LLM-as-a-judge to address the plausibility of the proposed outcomes. 
Due to the inference nature, the number of unique outcomes from MAELLE is much less than what comes out of Molecular Transformers or Graph2SMILES and thus we capped the number of predictions for test case to be the smallest number of predictions among the models.
Given a precursor set and a list of products, the 
LLM has to give a verdict of whether each of the products is plausible or not, without knowing which prediction comes from which model.
As a sanity check, we also include the ground truth product into this list and measure how much proportion of the ground truth is deemed as plausible by the LLM (see Appendix \ref{app:llm_judge} for more details). 
The plausibility rate is then calculated as the number of predictions classified as plausible by the LLMs divided by the total number of predictions according to top-k.

Figure \ref{fig:side_prod}a displays side-by-side top-k accuracy and top-k plausibility rate.
Agreeing with the results in Table \ref{tab:results}, top-k accuracy is ranked in the order of Graph2SMILES > Molecular Transformers > MAELLE.
However, top-1 plausibility rate of MAELLE (80.7\%) is significant higher than its top-1 accuracy (60.3\%), suggesting that incorrect predictions are not always implausible.
Furthermore, we observe a reversed order when it comes to top-2 and top-3 plausibility rates, with MAELLE being the highest, followed by Molecular Transformers and Graph2SMILES.
This hints at the fact that SMILES-based models are likely trained to overfit to one single outcome, which conflicts with the stochastic nature of chemical reactions.

For demonstration, Figure \ref{fig:side_prod}b shows an example where an incorrect prediction by MAELLE is still considered plausible by the LLM-judge.
In the test dataset, this reaction is a bromide substitution reaction involving an imidazole ring with two nucleophilic nitrogen N:10 and N:12.
The ground truth product features an attack of N:12 to C:21 as the main product, likely attributed to the fact that N:12 is less sterically hindered by the benzoyl group compared with N:10.
All three models manage to predict the correct ground truth as top-1.
However, intuitively speaking, the side product involved N:10 is not entirely impossible, and MAELLE was able to recover this minor outcome as its top-2 prediction.
On the other hand, due to the fact that Graph2SMILES and Molecular Transformers sample in the token space, top-2 predictions of these models are just hallucinated versions of their top-1 prediction.
This example is also aligned with the observation that AI systems have limited capability in writing SMILES and other symbolic sequences, which is one of the major bottlenecks in AI applications for chemistry \cite{edwards2022translation,jang2024can}.

Figure \ref{fig:side_prod}c showcases how the step-by-step CTMC of electron edits enables MAELLE to sample multiple outcomes that agree with chemical mechanisms.
The confidence scores assigned to each outcome successfully differentiate the major and minor products, agreeing with the ground truth.

\begin{figure}[t]
    \centering
    \includegraphics[width=\linewidth]{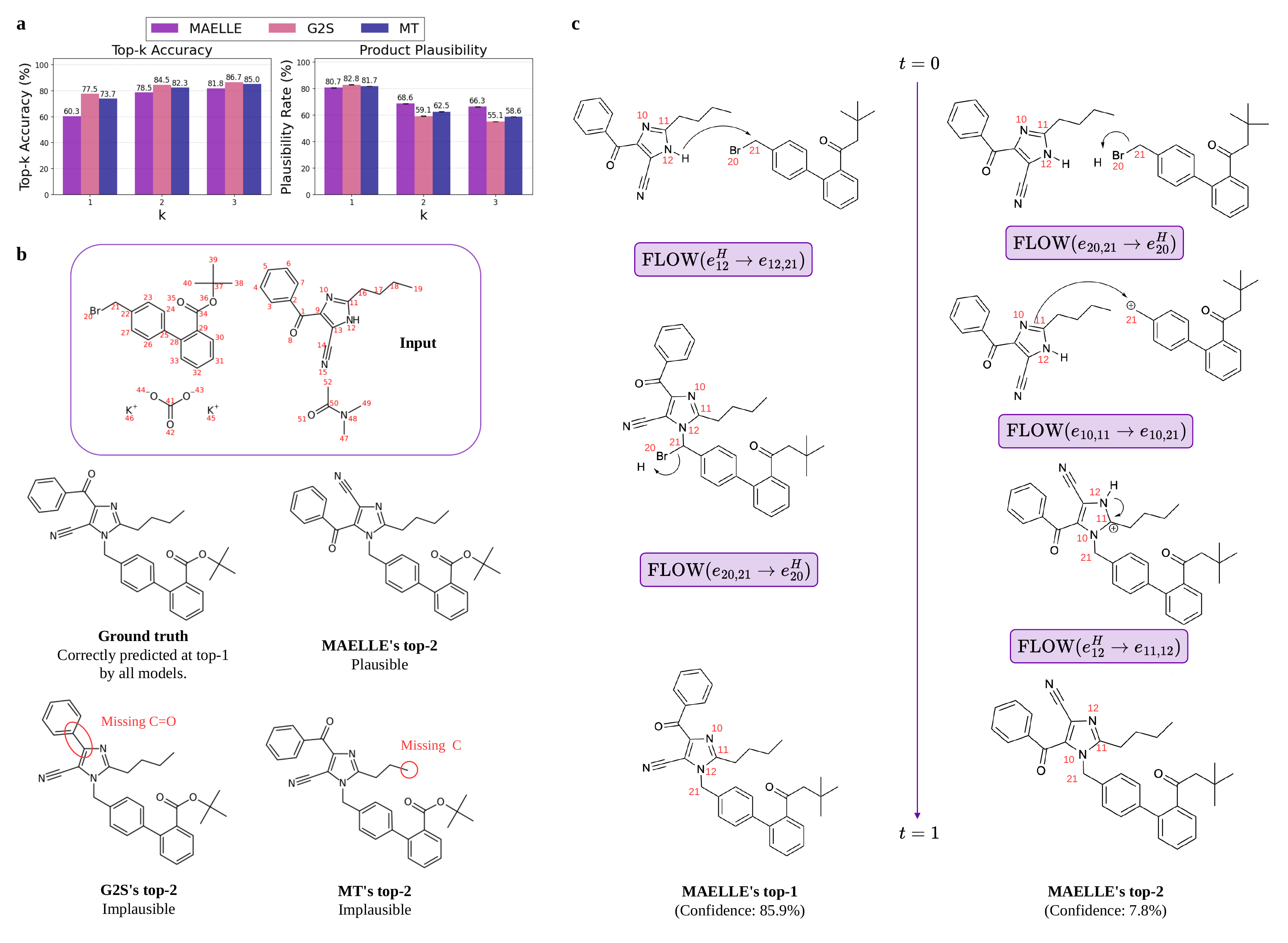}
    \caption{\textbf{a. }Top-k accuracy and plausibility rate, with the latter having mean and standard deviation reported for 3 runs of LLM evaluation. The LLM-judge has a ground truth coverage of $81.1 \pm 0.1\%$.\textbf{b. } Test case 39591 from the test set of USPTO-480K serves as an illustrative example where MAELLE successfully predicts the main product as top-1 and the side product as top-2, while Graph2SMILES and MolecularTransformers propose hallucinated outcomes. \textbf{c. }Two example trajectories that MAELLE samples leading to the two different outcomes.}
    \label{fig:side_prod}
\end{figure}


\subsection{Recovery of net directional electron flows}
\label{sec:directional}

\begin{table}[h]
\centering
\small
\caption{MAELLE's coverage of bond-forming elementary steps categorized into \textit{persistent} and \textit{transient} bonds.}
\label{tab:bond_recovery}
\begin{tabular}{@{}lllll@{}}
\toprule
\multirow{2}{*}{\textbf{\begin{tabular}[c]{@{}l@{}}Bond formation \\ type\end{tabular}}} & \multirow{2}{*}{\textbf{\begin{tabular}[c]{@{}l@{}}\# ground-truth \\ elementary steps\end{tabular}}} & \multicolumn{3}{c}{\textbf{Recovered by MAELLE's trajectories}} \\ \cmidrule(l){3-5} 
                                                                                         &                                                                                                       & \textbf{Correct}     & \textbf{Top 1}      & \textbf{All}       \\ \midrule
Persistent                                                                               & 17,925                                                                                                & 13,137 (73.3\%)      & 14,236 (79.4\%)     & 17,847 (99.6\%)    \\
Transient                                                                                & 19,658                                                                                                & 3,042 (15.5\%)       & 747 (3.8\%)         & 5,814 (29.6\%)     \\
Total                                                                                    & 37,583                                                                                                & 16,179 (43.0\%)      & 14,983 (39.9\%)     & 23,661 (63.0\%)    \\ \bottomrule
\end{tabular}
\end{table}
\begin{table}[t]
\centering
\small
\caption{MAELLE's coverage of bond-forming elementary steps and their directional agreement with FlowER's reference. For a MAELLE-predicted step and a ground-truth step that form the same bond, their directions agree if their sources and sinks are identical.}
\label{tab:direction}
\begin{tabular}{@{}lll@{}}
\toprule
\textbf{\begin{tabular}[c]{@{}l@{}}Sampling \\ mode\end{tabular}} & \textbf{\begin{tabular}[c]{@{}l@{}}Elementary steps \\ coverage\end{tabular}} & \textbf{\begin{tabular}[c]{@{}l@{}}Directional agreement \\ of elementary steps\end{tabular}} \\ \midrule
Correct                                                           & 16,179/37,583 (43.0\%)                                                        & 13,694/16,179 (84.6\%)                                                                        \\
Top-1                                                             & 14,983/37,583 (39.9\%)                                                        & 13,342/14,983 (89.0\%)                                                                        \\
All                                                               & 23,661/37,583 (63.0\%)                                                        & 18,505/23,661 (78.2\%)                                                                        \\ \bottomrule
\end{tabular}
\end{table}

\begin{figure}
    \centering
    \includegraphics[width=0.75\linewidth]{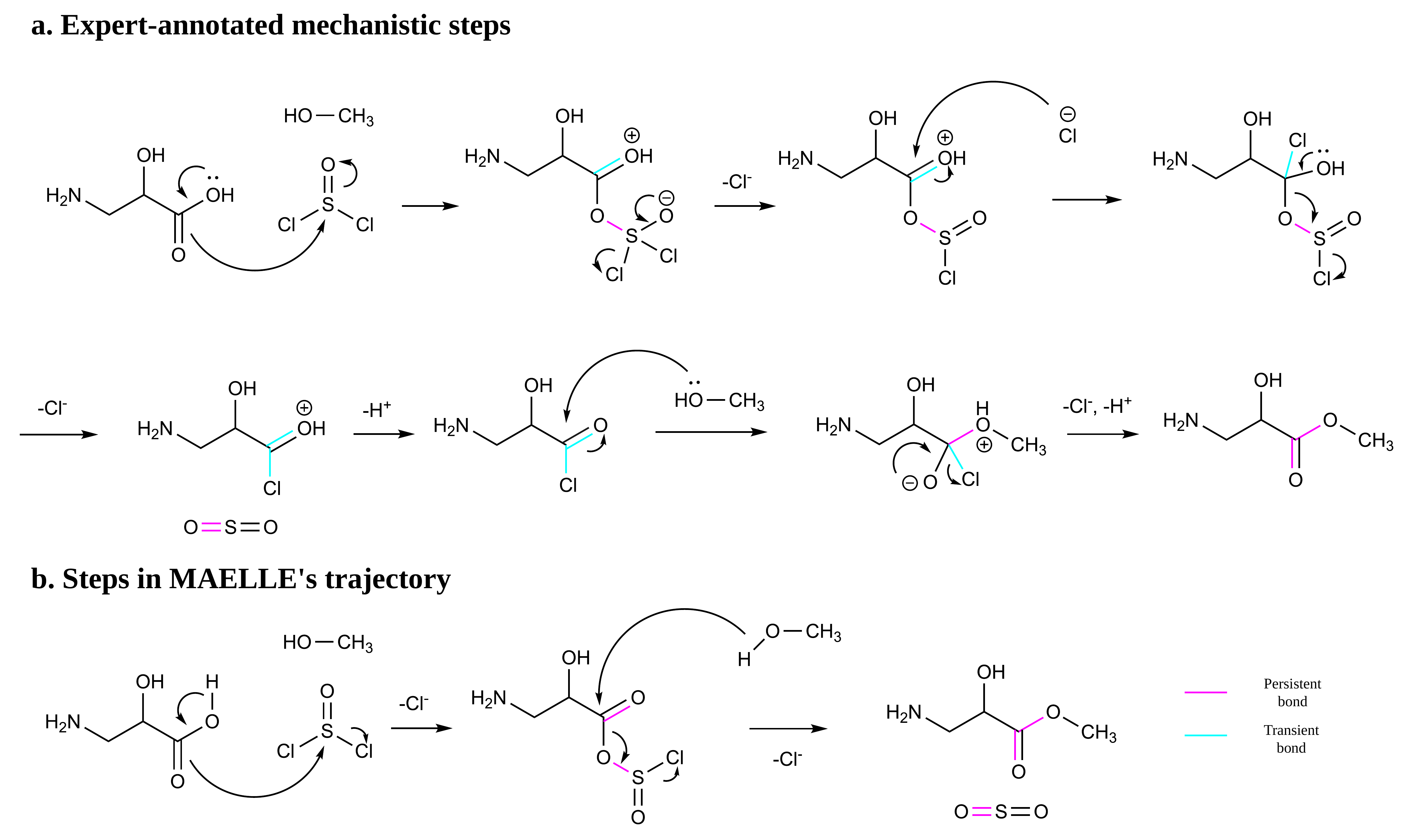}
    \caption{\textbf{a.} An example from FlowER's test set with expert-labeled elementary steps. \textbf{b.} Pseudo mechanistic steps in MAELLE's trajectory. Note that during training, MAELLE does not consume the labeled
    elementary steps.}
    \label{fig:flower}
\end{figure}

Beyond predicting the product, a further advantage of electron-based
trajectories is that they reveal \emph{how} the electrons are redistributed.
NERF also models reactions as electron redistributions, but because it predicts
only the net change at each bond, it cannot say where the electrons come from or
go. For MAELLE this information is explicit: the FLOW operation names a source
and a sink, so each trajectory carries a \emph{directional electron flow} that
resembles an arrow-pushing mechanism.


We train MAELLE on
the FlowER dataset~\citep{joung2025electron}, which contains $\sim$290K reactions extracted from
USPTO-Full~\citep{lowe2012extraction}, each annotated with a sequence of
expert-labeled elementary steps ($\sim$2M steps in total). MAELLE is trained
end-to-end and never consumes these labels. 
We draw $S=64$ trajectories per reaction and,
for each ground-truth bond-forming elementary step, check whether some
MAELLE trajectory forms the same bond and, if so, whether the source and sink of
its FLOW agree with the reference.


To support our claim that MAELLE can recover the \textit{directional electron flow}, i.e. bonds created as a result of the \textsc{FLOW} operation, we collect all bond-forming elementary steps in the test set of FlowER, and check how many of them are recovered from MAELLE's predictions, and check if the directions of the electron flow are correct, that is, when the source and sink are similar to the reference.

FlowER's test set contains $37{,}583$ bond-forming steps. We consider three ways
of selecting among the 64 trajectories: \textsc{correct} keeps only trajectories
that reach the ground-truth product; \textsc{top-1} takes the trajectories leading to the
highest-confidence product regardless of correctness;  and \textsc{all} pools all 64.
These recover $43.0\%$, $39.9\%$, and $63.0\%$ of the annotated steps,
respectively (Table~\ref{tab:bond_recovery}). The gap from full coverage is expected: expert pathways contain many \emph{transient} bonds
-- formed as short-lived intermediates and broken again before the product --
whereas MAELLE is trained on moves interpolated by optimal transport with a
topological cost (Sections~\ref{sec:ot}, \ref{sec:dfm}), which favors the
\emph{persistent} bonds that survive into the product.
Indeed, sampling all trajectories recovers $99.6\%$ of persistent bonds but only $29.6\%$
of transient ones (Table~\ref{tab:bond_recovery}).
Additionally, the coverage when we sample all trajectories is higher compared with sampling those of just the correct products or top-1, meaning that stochastic sampling surfaces additional annotated bonds, which contributes to the multi-outcome behavior in Section \ref{sec:side_products}.

Among the annotated steps that MAELLE recovers, we also observe high agreement in terms of the direction of the electron moves (Table~\ref{tab:direction}).
For \textsc{correct}, \textsc{top-1}, and \textsc{all} sampling schemes, the directional agreement rates are 84.6\%, 89.0\%, and 78.2\%, respectively, showing that the electron moves produced by MAELLE, which is an artifact of the edit-based mixture path (Section~\ref{sec:dfm}), serve as a good approximation of the labeled steps.

Figure~\ref{fig:flower} shows a representative case from FlowER's test set: the
esterification of a carboxylic acid activated by thionyl chloride
($\mathrm{SOCl_2}$).
The expert pathway proceeds through a chlorosulfite intermediate -- the
carboxylic oxygen attacks sulfur, chloride leaves then attacks the carboxylic carbon, displacing the chlorosulfite ester. The chloride itself ultimately becomes the leaving group for the substitution of methanol.
Throughout the trajectory, several transient bonds are formed, including C-Cl and one electron pair of the C=O bond.
On the other hand, MAELLE's trajectory recovers the persistent bond formations that define the
product -- the new acyl C-O bond to the incoming alkoxy group, the change from C-O single bond to C=O double bond, and the formation of the S=O double bond.
The predicted flows are directionally consistent with the
reference on the persistent steps, illustrating both the coverage and the
directional-agreement statistics above.

Together these results support our claim that MAELLE's trajectories are
\emph{pseudo-mechanistic}: without any elementary-step supervision, they recover
the persistent, product-defining electron flows and their direction, serving as
an approximation of expert-annotated mechanisms.

\section{Conclusion}

We introduced MAELLE, a reaction prediction model that operates in the electron occupation space via discrete flow matching. By formulating reactions as CTMCs over an auxiliary move-state space derived from optimal transport, MAELLE bridges the gap between machine learning and the arrow-pushing formalism that chemists use to reason about reactivity. Our experiments demonstrate three key findings: (i) competitive in-distribution performance on USPTO-480K despite operating in a more constrained representation than SMILES-based methods; (ii) strong out-of-distribution robustness on both molecular-weight and reaction-type generalization, where MAELLE outperforms all baselines; and (iii) the ability to predict chemically plausible side products through multiple CTMC trajectory sampling, a capability inaccessible to deterministic or beam-search-based approaches.

\paragraph{Limitations and future work.}
MAELLE currently models electron pairs rather than individual electrons, which precludes radical reactions involving unpaired electrons; extending the occupation space to single-electron resolution would address this. The framework predicts forward reactions only; adapting the edit operations to the atom level (e.g., atom insertion and deletion) would enable retrosynthetic prediction. While the OT-derived move trajectories are mechanistically interpretable and correlate with known chemistry, they represent pseudo-mechanisms rather than energy-optimal elementary steps -- the model learns the most likely \emph{combinatorial} path between reactant and product occupations, not the true potential energy surface. Incorporating energetic priors or non-optimal-transport pathways could improve mechanistic fidelity. We note that reagents and catalysts are already incorporated as context through the reactant graph $\mathcal{G}_0$; the model conditions on them even though they do not directly participate in the electron flow.


\bibliography{citation}


\appendix

\section{Equivalence of the edit-based and augmented-space mixture paths}
\label{app:equiv}

We show, in a setting-agnostic form, that the augmented-space mixture path of
Edit Flows~\citep{havasi2025edit} is equivalent to directly sampling a
subset of edit moves and applying them. 
The statement is
independent of what the edits are: it applies verbatim to sequence edits
(insertion, deletion, substitution) and to the electron-occupation edits
(\textsc{Flow}, \textsc{Del}, \textsc{Add}) used in the main text.

\paragraph{Setup.}
Let $\mathcal{D}$ be a state space (e.g.\ token sequences, or electron
occupation vectors) and fix a source state $x_0\in\mathcal{D}$. Let
$\mathcal{M}=\{m_1,\dots,m_K\}$ be a finite set of edit moves, each an operation
on $\mathcal{D}$, and let
\begin{equation}
  \mathrm{Apply}:\mathcal{D}\times 2^{\mathcal{M}}\to\mathcal{D}
\end{equation}
apply a subset of moves to a state, assumed well defined (independent of the
order in which the moves in the subset are applied). Write $x_1 =
\mathrm{Apply}(x_0,\mathcal{M})$ for the state obtained by applying all moves.
Let $\kappa:[0,1]\to[0,1]$ be a schedule with $\kappa_0=0$ and $\kappa_1=1$.

We compare two ways of building a conditional path
$p_t(\,\cdot\mid x_0,x_1)$ on $\mathcal{D}$ that interpolates from $x_0$ at
$t=0$ to $x_1$ at $t=1$.

\paragraph{(A) Augmented-space path.}
Introduce the auxiliary space $\mathcal{Z}=\{0,1\}^{K}$, one binary coordinate
per move, with the readout
\begin{equation}
  g:\mathcal{Z}\to\mathcal{D},
  \qquad
  g(z)=\mathrm{Apply}\bigl(x_0,\{m_k : z_k=1\}\bigr).
  \label{eq:app-readout}
\end{equation}
Treating each coordinate as an independent binary token that interpolates
$z_k:0\!\to\!1$~\citep{campbell2024generative,gat2024discrete} gives the
factorized path on $\mathcal{Z}$
\begin{equation}
  p_t(z\mid x_0,x_1)=\prod_{k=1}^{K}
  \Bigl[(1-\kappa_t)\,\delta_0(z_k)+\kappa_t\,\delta_1(z_k)\Bigr],
  \label{eq:app-auxpath}
\end{equation}
and the augmented-space path on $\mathcal{D}$ is its pushforward through $g$,
\begin{equation}
  p^{\mathrm{A}}_t(x\mid x_0,x_1)
  =\sum_{z\in\mathcal{Z}} p_t(z\mid x_0,x_1)\,\mathbbm{1}[g(z)=x].
  \label{eq:app-A}
\end{equation}
For sequence edits with $\mathcal{Z}$ realized as padded sequences and $g$ as
padding removal, \eqref{eq:app-A} is the mixture path of
\citet{havasi2025edit}.

\paragraph{(B) Edit-based path.}
Include each move independently with probability $\kappa_t$, collect the applied
subset $\mathcal{S}$, and apply it to $x_0$. The resulting law is
\begin{equation}
  p^{\mathrm{B}}_t(x\mid x_0,x_1)
  =\sum_{\mathcal{S}\subseteq\mathcal{M}}
  \kappa_t^{\,|\mathcal{S}|}(1-\kappa_t)^{\,K-|\mathcal{S}|}\,
  \mathbbm{1}[\mathrm{Apply}(x_0,\mathcal{S})=x].
  \label{eq:app-B}
\end{equation}

\begin{proposition}[Edit-sampling form of the augmented-space path]
\label{prop:equiv}
For every $t\in[0,1]$ and every $x\in\mathcal{D}$,
\begin{equation}
  p^{\mathrm{A}}_t(x\mid x_0,x_1)=p^{\mathrm{B}}_t(x\mid x_0,x_1).
\end{equation}
Consequently sampling $z\sim p_t(\cdot\mid x_0,x_1)$ from \eqref{eq:app-auxpath}
and returning $g(z)$ has the same law as sampling a subset $\mathcal{S}$ of
moves by independent $\mathrm{Bernoulli}(\kappa_t)$ inclusion and returning
$\mathrm{Apply}(x_0,\mathcal{S})$. In particular $p^{\mathrm{A}}_0=p^{\mathrm{B}}_0=\delta_{x_0}$
and $p^{\mathrm{A}}_1=p^{\mathrm{B}}_1=\delta_{x_1}$, so both are valid
conditional interpolants.
\end{proposition}

\begin{proof}
The map $z\mapsto \mathcal{S}(z):=\{m_k : z_k=1\}$ is a bijection between
$\mathcal{Z}=\{0,1\}^K$ and the power set $2^{\mathcal{M}}$, with inverse
$\mathcal{S}\mapsto z$ where $z_k=\mathbbm{1}[m_k\in\mathcal{S}]$. We rewrite the
augmented-space weight \eqref{eq:app-auxpath} under this bijection.

Each factor in \eqref{eq:app-auxpath} selects one of its two terms according to
$z_k$: it contributes $\kappa_t$ when $z_k=1$ and $1-\kappa_t$ when $z_k=0$.
Hence for a fixed $z$,
\begin{equation}
  p_t(z\mid x_0,x_1)
  =\prod_{k=1}^{K}\kappa_t^{\,z_k}(1-\kappa_t)^{\,1-z_k}
  =\kappa_t^{\sum_k z_k}\,(1-\kappa_t)^{\,K-\sum_k z_k}
  =\kappa_t^{\,|\mathcal{S}(z)|}(1-\kappa_t)^{\,K-|\mathcal{S}(z)|},
  \label{eq:app-weight}
\end{equation}
using $\sum_k z_k=|\mathcal{S}(z)|$. By definition of the readout
\eqref{eq:app-readout}, $g(z)=\mathrm{Apply}(x_0,\mathcal{S}(z))$. Substituting
both identities into \eqref{eq:app-A} and reindexing the sum from $z$ to
$\mathcal{S}=\mathcal{S}(z)$,
\begin{align}
  p^{\mathrm{A}}_t(x\mid x_0,x_1)
  &=\sum_{z\in\mathcal{Z}}
    \kappa_t^{\,|\mathcal{S}(z)|}(1-\kappa_t)^{\,K-|\mathcal{S}(z)|}\,
    \mathbbm{1}[\mathrm{Apply}(x_0,\mathcal{S}(z))=x] \notag\\
  &=\sum_{\mathcal{S}\subseteq\mathcal{M}}
    \kappa_t^{\,|\mathcal{S}|}(1-\kappa_t)^{\,K-|\mathcal{S}|}\,
    \mathbbm{1}[\mathrm{Apply}(x_0,\mathcal{S})=x]
  = p^{\mathrm{B}}_t(x\mid x_0,x_1).
\end{align}
This proves the equality for all $t$ and $x$.

For the boundary conditions, at $t=0$ we have $\kappa_0=0$, so the weight
$\kappa_0^{|\mathcal{S}|}(1-\kappa_0)^{K-|\mathcal{S}|}$ is nonzero only for
$\mathcal{S}=\varnothing$, where it equals $1$; since
$\mathrm{Apply}(x_0,\varnothing)=x_0$, we get $p_0=\delta_{x_0}$. At $t=1$,
$\kappa_1=1$, so the weight is nonzero only for $\mathcal{S}=\mathcal{M}$, where
it equals $1$; since $\mathrm{Apply}(x_0,\mathcal{M})=x_1$, we get
$p_1=\delta_{x_1}$.
\end{proof}


\begin{corollary}[Sequence edits]
\label{cor:text}
Taking $\mathcal{D}$ to be token sequences and $\mathcal{M}$ a set of insertion,
deletion, and substitution moves aligning $x_0$ to $x_1$, Proposition~\ref{prop:equiv}
recovers the augmented-space mixture path of Edit
Flows~\citep{havasi2025edit}, with $\mathcal{Z}$ realized as padded
sequences and $g=f_\mathrm{rm-blanks}$ that strips all padding tokens.
\end{corollary}

\begin{corollary}[Electron edits]
\label{cor:electron}
Taking $\mathcal{D}=\mathbb{Z}^{|\mathcal{E}|}_{\ge 0}$, $x_0=o_x$,
$\mathcal{M}$ the move set of Eq.~\eqref{eq:moveset}, and
$\mathrm{Apply}=\mathrm{ApplyEdit}$, Proposition~\ref{prop:equiv} shows that the
edit-based mixture path \eqref{eq:edit-mixture} equals the augmented-space
pushforward, justifying the interpolant-sampling scheme of
Algorithm~\ref{alg:sample}.
\end{corollary}


\section{Problem-statement details}
\label{app:problem_statement}


\paragraph{Atoms, adjacency, and electron sites.}
$\mathcal{V} = \{v_i \mid 1 \leq i \leq |\mathcal{V}|\}$ is the set of heavy atoms; hydrogens are not nodes but are pooled into per-atom electron sites. The adjacency $\mathcal{A} = \{a_{ij} \in \{0,1\} \mid 1 \leq i < j \leq |\mathcal{V}|\}$ follows the standard molecular graph convention~\cite{duval2023hitchhiker}. The electron-site set
$$
\mathcal{E} = \{e_{ij} \mid 1 \leq i \leq j \leq |\mathcal{V}|\} \cup \{e^H_i \mid 1 \leq i \leq |\mathcal{V}|\}
$$
has $|\mathcal{E}| = \binom{|\mathcal{V}|}{2} + 2|\mathcal{V}|$ entries: bonding sites $e_{ij}$ ($i \neq j$), lone-pair sites $e_{ii}$, and per-atom hydrogen sites $e^H_i$. Hydrogens bonded to the same heavy atom are topologically equivalent so a single $e^H_i$ suffices~\cite{chen2025predicting}. Each site is either \emph{realized} (existing bond/lone pair/hydrogen attachment) or \emph{virtual} ($o(e) = 0$, available to become realized).

\paragraph{Occupation $\Rightarrow$ topology.}
The map $f_{o \to a}: (\mathcal{V}, \mathcal{E}, o) \mapsto \mathcal{A}$ deterministically recovers the heavy-atom adjacency from the occupation via Eq.~\eqref{eq:o_to_a}. Combined with the fact that $\mathcal{V}$ (and hence $\mathcal{E}$) is preserved by a reaction, this yields the simplification
$$
p(\mathcal{G}_y \mid \mathcal{G}_x)
= p(\mathcal{V}_y, \mathcal{A}_y, \mathcal{E}_y, o_y \mid \mathcal{V}_x, \mathcal{A}_x, \mathcal{E}_x, o_x)
= p(o_y \mid \mathcal{V}_x, \mathcal{A}_x, \mathcal{E}_x, o_x).
$$

\paragraph{Ethanol example.}
Consider ethanol (CH$_3$CH$_2$OH) with heavy atoms $\mathcal{V} = \{\text{C}_1, \text{C}_2, \text{O}_3\}$. The electron sites and their occupations are:
\begin{center}
\small
\begin{tabular}{@{}lccl@{}}
\toprule
\textbf{Site} & \textbf{Electron type} & \textbf{Occupation} $o$ & \textbf{Description} \\
\midrule
$e_{1,2}$ (C$_1$--C$_2$) & Bond & 1 & Single bond \\
$e_{2,3}$ (C$_2$--O) & Bond & 1 & Single bond \\
$e_{1,3}$ (C$_1$--O) & Virtual & 0 & Virtual (no bond) \\
$e_{3,3}$ (O lone pairs) & Valence & 2 & Two lone pairs \\
$e_{1,1}$, $e_{2,2}$ & Virtual & 0 & No lone pairs on carbons \\
$e^H_1$ (C$_1$) & w/ Hydrogens & 3 & Three hydrogens \\
$e^H_2$ (C$_2$) & w/ Hydrogens & 2 & Two hydrogens \\
$e^H_3$ (O) & w/ Hydrogens & 1 & One hydrogen \\
\bottomrule
\end{tabular}
\end{center}
This gives a total of 10 occupied electron pairs across 9 sites, consistent with ethanol's 20 valence electrons.

\section{Optimal Transport Details}
\label{app:ot_details}

To handle reactions where the total electron count is not conserved, we augment the electron site set with a virtual site $e_\varnothing$ acting as a source and sink for electron pairs gained or lost. Let $\bar{\mathcal{E}} = \mathcal{E} \cup \{e_\varnothing\}$ and define augmented occupation vectors $\bar{o}_x, \bar{o}_y \in \mathbb{Z}_{\geq 0}^{|\bar{\mathcal{E}}|}$:
$$
    \bar{o}_x(e) =
    \begin{cases}
        o_x(e) & \text{if } e \in \mathcal{E}, \\
        \sum_{e' \in \mathcal{E}} \max(o_y(e') - o_x(e'),\, 0) & \text{if } e = e_\varnothing,
    \end{cases}
$$
$$
    \bar{o}_y(e) =
    \begin{cases}
        o_y(e) & \text{if } e \in \mathcal{E}, \\
        \sum_{e' \in \mathcal{E}} \max(o_x(e') - o_y(e'),\, 0) & \text{if } e = e_\varnothing,
    \end{cases}
$$
so that $\sum_{e \in \bar{\mathcal{E}}} \bar{o}_x(e) = \sum_{e \in \bar{\mathcal{E}}} \bar{o}_y(e)$, ensuring balanced transport.

\paragraph{Cost matrix.}
We define the cost matrix $C \in \mathbb{Z}_{\geq 0}^{|\bar{\mathcal{E}}| \times |\bar{\mathcal{E}}|}$ based on shortest-path distances. Let $d_x(v, v')$ and $d_y(v, v')$ denote shortest-path distances on the reactant and product graphs, respectively. Since bonds may form or break, we take the minimum over both:
$$
    C_{e,e'} = \min\!\left(
        \min_{\substack{v \in \text{atoms}(e) \\ v' \in \text{atoms}(e')}} d_x(v, v'),\;
        \min_{\substack{v \in \text{atoms}(e) \\ v' \in \text{atoms}(e')}} d_y(v, v')
    \right), \qquad \forall\, e, e' \in \mathcal{E},
$$
where $\text{atoms}(e) \subseteq \mathcal{V}$ denotes the atoms associated with site $e$:
$$
    \text{atoms}(e) =
    \begin{cases}
        \{v_i, v_j\} & \text{if } e = e_{i,j} \text{ with } i \neq j, \\
        \{v_i\} & \text{if } e = e_{i,i} \text{ or } e = e^H_i.
    \end{cases}
$$
The cost to or from the virtual site is set to a fixed large constant $C_{e,e_\varnothing} = C_{e_\varnothing,e} = c_\varnothing$, incentivizing the solver to use $e_\varnothing$ only when necessary.

\section{Algorithms}
\label{app:algorithms}

\begin{algorithm}[H]
\caption{\proc{SampleIntermediateGraph}$(\mathcal{G}_0,\; \mathbf{M},\; \kappa)$}\label{alg:sample}
\begin{algorithmic}[1]
\Require Reactant graph $\mathcal{G}_0$, set of electron moves $\mathbf{M}$, interpolation level $\kappa \in [0,1]$
\Ensure Intermediate graph $\mathcal{G}_t$, remaining target moves $\mathbf{M}_{\mathrm{target}}$
\Statex
\For{each move $(e^s, e^d) \in \mathbf{M}$}
    \State Sample $z_{s,d} \sim \mathrm{Bernoulli}(\kappa)$
\EndFor
\State $\mathbf{M}_{\mathrm{fired}} \gets \{(e^s, e^d) : z_{s,d}=1\}$ \hfill $\triangleright$ moves already applied
\State $\mathbf{M}_{\mathrm{target}} \gets \{(e^s, e^d) : z_{s,d}=0\}$ \hfill $\triangleright$ moves still to predict
\State $\mathcal{G}_t \gets \proc{ApplyElectronEdits}(\mathcal{G}_0,\; \mathbf{M}_{\mathrm{fired}})$
\State \Return $\mathcal{G}_t,\; \mathbf{M}_{\mathrm{target}}$
\end{algorithmic}
\end{algorithm}
 
\vspace{1em}
 
\begin{algorithm}[H]
\caption{\textbf{MAELLE}'s Training Loop}\label{alg:train}
\begin{algorithmic}[1]
\Require Data distribution $p(\mathbf{x}, \mathbf{y})$, model $u_\theta$, kappa schedule $\kappa(\cdot)$
\Statex
\State Sample $(\mathcal{G}_0,\; \mathcal{G}_1) \sim p(\mathbf{x}, \mathbf{y})$
\hfill $\triangleright$ reactant and product graphs
\State $\mathbf{M} \gets \proc{SolveOptimalTransport}(\mathcal{G}_0,\; \mathcal{G}_1)$
\hfill $\triangleright$ electron moves from OT plan
\State Sample $t \sim \mathrm{Uniform}(0,1)$
\State $\mathcal{G}_t,\; \mathbf{M}_{\mathrm{target}} \gets \proc{SampleIntermediateGraph}(\mathcal{G}_0,\; \mathbf{M},\; \kappa(t))$
\hfill $\triangleright$ Alg.~\ref{alg:sample}
 
\Statex
\State \textbf{// Forward pass}
\State $\boldsymbol{\lambda}^{\mathrm{flow}},\; \boldsymbol{\lambda}^{\mathrm{del}},\; \boldsymbol{\lambda}^{\mathrm{add}},\; \mathcal{E} \gets u_\theta(\mathcal{G}_t,\; \mathcal{G}_0,\; t)$
\Statex \hfill $\triangleright$ per-site flow/del rates over $\mathcal{E}$; per-atom add rates over $\mathcal{V}$
 
\Statex
\State \textbf{// Bregman divergence loss}
\State Compute $\mathcal{L}$ and update $\theta$ via back-propagation:
\begin{align}
\mathcal{L} \;=\; \sum_{e \in \mathcal{E}} \bigl(\lambda_e^{\mathrm{flow}} + \lambda_e^{\mathrm{del}}\bigr) + \sum_{a \in \mathcal{V}} \lambda_a^{\mathrm{add}}
\;-\; \frac{\kappa'(t)}{1 - \kappa(t)} \!\!\!\!\sum_{(e^s,\,e^d)\,\in\,\mathbf{M}_{\mathrm{target}}} \!\!\!\!
\begin{cases}
    \log \lambda_a^{\mathrm{add}} + \log p^{\text{sink}}_\theta(e^d \mid a,\; \mathcal{E})
    & \text{if } e^s = \varnothing \\[4pt]
    \log \lambda_{e^s}^{\mathrm{del}}
    & \text{if } e^d = \varnothing \\[4pt]
    \log \lambda_{e^s}^{\mathrm{flow}} + \log p^{\text{sink}}_\theta(e^d \mid e^s,\; \mathcal{E})
    & \text{otherwise}
\end{cases} \notag
\end{align}
\end{algorithmic}
\end{algorithm}
 
\vspace{1em}

 
\begin{algorithm}[H]
\caption{\proc{SampleMoves}$(\mathcal{G}_t,\; \boldsymbol{\lambda}^{\mathrm{flow}},\; \boldsymbol{\lambda}^{\mathrm{del}},\; \boldsymbol{\lambda}^{\mathrm{add}},\; \mathcal{E},\; \Delta t,\; \tau)$}\label{alg:sample_moves}
\begin{algorithmic}[1]
\Require Rates $\boldsymbol{\lambda}^{\mathrm{flow}},\boldsymbol{\lambda}^{\mathrm{del}}$ (per site), $\boldsymbol{\lambda}^{\mathrm{add}}$ (per atom), electron sites $\mathcal{E}$, step size $\Delta t$, temperature $\tau$
\Ensure Set of moves $\mathbf{M}_{\mathrm{step}}$
\Statex
 
\State \textbf{// Fire sites via thinned Poisson process}
\For{each site $e \in \mathcal{E}$}
    \State $p_e^{\mathrm{site}} \gets 1 - \exp\!\bigl(-(\lambda_e^{\mathrm{flow}} + \lambda_e^{\mathrm{del}})\,\Delta t \,/\, \tau \bigr)$
    \State Sample $f_e \sim \mathrm{Bernoulli}(p_e^{\mathrm{site}})$
\EndFor
\State $\mathcal{F} \gets \{e : f_e = 1\}$
\hfill $\triangleright$ fired sites
 
\Statex
\State \textbf{// Fire addition atoms}
\For{each atom $a \in \mathcal{V}$}
    \State $p_a^{\mathrm{add}} \gets 1 - \exp\!\bigl(-\lambda_a^{\mathrm{add}}\,\Delta t \,/\, \tau\bigr)$
    \State Sample $f_a \sim \mathrm{Bernoulli}(p_a^{\mathrm{add}})$
\EndFor
\State $\mathcal{A}_{\mathrm{add}} \gets \{a : f_a = 1\}$
\hfill $\triangleright$ fired addition atoms
 
\Statex
\State \textbf{// Resolve actions for fired sites}
\State $\mathbf{M}_{\mathrm{flow}},\;\mathbf{M}_{\mathrm{del}} \gets \varnothing,\; \varnothing$
\For{each $e^{\text{fired}} \in \mathcal{F}$}
    \State $\text{action} \gets
    \begin{cases}
        \textit{flow} & \text{w.p. } \lambda_e^{\mathrm{flow}} / (\lambda_e^{\mathrm{flow}} + \lambda_e^{\mathrm{del}}) \\[4pt]
        \textit{del} & \text{otherwise}
    \end{cases}$
    \If{action $=$ \textit{flow}}
        \State Sample sink $e^d \sim p^{\text{sink}}_\theta(\,\cdot \mid e^{\text{fired}},\; \mathcal{E})$
        \State $\mathbf{M}_{\mathrm{flow}} \gets \mathbf{M}_{\mathrm{flow}} \cup \{(e^{\text{fired}}, e^d)\}$
    \Else
        \State $\mathbf{M}_{\mathrm{del}} \gets \mathbf{M}_{\mathrm{del}} \cup \{(e^{\text{fired}}, \varnothing)\}$
    \EndIf
\EndFor
\State Resolve conflicts in $\mathbf{M}_{\mathrm{flow}}$: keep highest $\lambda_e^{\mathrm{flow}}$ per destination
 
\Statex
\State \textbf{// Resolve actions for fired additions}
\State $\mathbf{M}_{\mathrm{add}} \gets \varnothing$
\For{each $a \in \mathcal{A}_{\mathrm{add}}$}
    \State Sample sink $e^d \sim p^{\text{sink}}_\theta(\,\cdot \mid a,\; \mathcal{E})$
    \State $\mathbf{M}_{\mathrm{add}} \gets \mathbf{M}_{\mathrm{add}} \cup \{(\varnothing, e^d)\}$
\EndFor
 
\Statex
\State \Return $\mathbf{M}_{\mathrm{step}} \gets \mathbf{M}_{\mathrm{flow}} \cup \mathbf{M}_{\mathrm{del}} \cup \mathbf{M}_{\mathrm{add}}$
\end{algorithmic}
\end{algorithm}
 
\vspace{1em}
 
\begin{algorithm}[H]
\caption{\textbf{MAELLE}'s Inference}\label{alg:inference}
\begin{algorithmic}[1]
\Require Reactant graph $\mathcal{G}_0$, trained model $u_\theta$, number of steps $N$, temperature $\tau$
\Ensure Predicted product graph $\hat{\mathcal{G}}_1$
\Statex
 
\State $\mathcal{G}_t \gets \mathcal{G}_0$, \quad $\Delta t \gets 1/N$
\State Encode reactant: $\mathbf{z}_0 \gets \proc{Encode}_{u_\theta}(\mathcal{G}_0)$
\hfill $\triangleright$ computed once, reused at every step
 
\Statex
\For{$n = 0, \dots, N-1$}
    \State $t \gets n \,/\, N$
    \State $\Delta t_{\mathrm{adapt}} \gets \proc{AdaptiveStepSize}(\Delta t,\; t,\; \kappa)$
    \hfill $\triangleright$ adaptive $h$ from schedule
 
    \Statex
    \State \textbf{// Decode current state}
    \State $\boldsymbol{\lambda}^{\mathrm{flow}},\; \boldsymbol{\lambda}^{\mathrm{del}},\; \boldsymbol{\lambda}^{\mathrm{add}},\; \mathcal{E} \gets \proc{Decode}_{u_\theta}(\mathcal{G}_t,\; \mathbf{z}_0,\; t)$
    \Statex \hfill $\triangleright$ per-site flow/del rates; per-atom add rates; $\mathcal{E}$: electron sites in $\mathcal{G}_t$
 
    \Statex
    \State \textbf{// Sample and apply moves}
    \State $\mathbf{M}_{\mathrm{step}} \gets \proc{SampleMoves}(\mathcal{G}_t,\; \boldsymbol{\lambda}^{\mathrm{flow}},\; \boldsymbol{\lambda}^{\mathrm{del}},\; \boldsymbol{\lambda}^{\mathrm{add}},\; \mathcal{E},\; \Delta t_{\mathrm{adapt}},\; \tau)$
    \hfill $\triangleright$ Alg.~\ref{alg:sample_moves}
    \State $\mathcal{G}_t \gets \proc{ApplyElectronEdits}(\mathcal{G}_t,\; \mathbf{M}_{\mathrm{step}})$
\EndFor
 
\Statex
\State \Return $\hat{\mathcal{G}}_1 \gets \mathcal{G}_t$
\end{algorithmic}
\end{algorithm}



\section{Neural Network Architecture}
\label{app:architecture}

MAELLE uses a conditional encoder-decoder architecture. The \emph{encoder} processes the fixed reactant graph $\mathcal{G}_0$ to produce conditioning representations; the \emph{decoder} processes the time-dependent intermediate graph $\mathcal{G}_t$ and cross-attends to the encoder outputs. Both operate at two levels of granularity: atoms and electron sites (entries).

\paragraph{Notation.}
Let $d$ denote the model dimension, $N$ the number of heavy atoms in a graph, and $N_e = \binom{N}{2} + 2N$ the number of electron sites (valence, hydrogen, and bond entries). We write $\mathbf{h} \in \mathbb{R}^{N \times d}$ for atom embeddings and $\mathbf{e} \in \mathbb{R}^{N_e \times d}$ for entry embeddings.

\subsection{Graph encoder (GNN backbone)}

Both the encoder and decoder share the same GNN architecture but with separate weights. The backbone is a stack of $L$ Graph Attention Network (GAT) layers~\cite{velivckovic2017graph}. Intermediate layers use multi-head attention with 4 heads and concatenation (output dimension $4 \times d_{\text{hidden}}$), while the final layer uses a single head without concatenation (output dimension $d$). An ELU activation is applied between intermediate layers.

Input atom features are first projected via a linear layer $\mathbf{h}^{(0)} = W_{\text{node}} \mathbf{x} + \mathbf{b}_{\text{node}}$, where $\mathbf{x} \in \mathbb{R}^{N \times d_{\text{in}}}$ are the raw atom features.

\paragraph{Time embedding.}
The decoder's GNN incorporates the flow time $t \in [0, 1]$ via a sinusoidal positional embedding followed by a two-layer MLP with SiLU activation:
$$
    \mathbf{t}_{\text{emb}} = \mathrm{MLP}_{\text{time}}(\mathrm{SinEmb}(t)) \in \mathbb{R}^{d_{\text{hidden}}},
$$
which is broadcast-added to the initial node embeddings $\mathbf{h}^{(0)}$ before the GAT layers. The encoder's GNN omits the time embedding, as the reactant graph $\mathcal{G}_0$ is fixed.

\subsection{Two-level attention}

After the GNN backbone, both encoder and decoder refine representations through two stages of attention operating at different granularities.

\paragraph{Stage 1: Atom-level attention.}
The atom embeddings $\mathbf{h} \in \mathbb{R}^{N \times d}$ from the GNN are passed through a Transformer. In the \emph{encoder}, this is a TransformerEncoder (self-attention only). In the \emph{decoder}, this is a TransformerDecoder: each layer applies self-attention over the decoder's atom embeddings, followed by cross-attention to the encoder's atom embeddings. This allows the decoder to condition on the reactant atom representations at every layer.

\paragraph{Stage 2: Entry-level feature construction and attention.}
Atom embeddings are lifted to electron-site (entry) embeddings via three learned projections:
\begin{itemize}
    \item \textbf{Valence entries} (lone pairs): $\mathbf{e}^{\text{val}}_i = \mathrm{ReLU}(W_{\text{val}}[\mathbf{h}_i \,\|\, \phi(n^{\text{nb}}_i)])$ for each atom $i$, where $n^{\text{nb}}_i$ is the non-bonding electron count and $\phi(\cdot)$ is a 4-bin one-hot encoding of the electron pair count (0, 1, 2, $\geq$3 pairs).
    \item \textbf{Hydrogen entries}: $\mathbf{e}^{H}_i = \mathrm{ReLU}(W_{H}[\mathbf{h}_i \,\|\, \phi(n^{H}_i)])$ for each atom $i$, where $n^{H}_i$ is the hydrogen electron count.
    \item \textbf{Bond entries}: $\mathbf{e}^{\text{bond}}_{ij} = \mathrm{ReLU}(W_{\text{bond}}[\mathrm{agg}(\mathbf{h}_i, \mathbf{h}_j) \,\|\, \rho(d_{ij}) \,\|\, \phi(n^{\text{bond}}_{ij})])$ for each realized bond $(i,j)$, where $\mathrm{agg}(\cdot,\cdot)$ is a symmetric aggregation (sum, mean, Hadamard, or bilinear), $\rho(d_{ij})$ is a one-hot encoding of the topological distance (up to 10 hops), and $n^{\text{bond}}_{ij}$ is the bond electron count.
\end{itemize}
These three sets of entry embeddings are interleaved into a single sequence of $N_e$ entries per graph and passed through another Transformer block --- self-attention in the encoder, self-attention plus cross-attention to the encoder's entry embeddings in the decoder.

\subsection{Output heads}

The decoder produces three types of predictions from the entry embeddings, corresponding to the source--sink factorization in Eq.~\eqref{eq:source_sink}:

\paragraph{Rate head.}
A 3-layer MLP (with ReLU activations and hidden dimension $2d$) maps each entry embedding to three non-negative scalars via softplus:
$$
    [\lambda_{\theta,t}^{\mathrm{flow}}(e), \; \lambda_{\theta,t}^{\mathrm{del}}(e), \; \lambda_{\theta,t}^{\mathrm{add}}(e)] = \mathrm{softplus}(\mathrm{MLP}_{\text{rate}}(\mathbf{e})).
$$
The addition rate $\lambda_{\theta,t}^{\mathrm{add}}$ is only used for atom-level entries (valence sites); for bond entries it is masked out.

\paragraph{Sink heads.}
Two separate MLPs parameterize the sink distributions $p^{\mathrm{sink,flow}}_\theta$ and $p^{\mathrm{sink,add}}_\theta$. Given a fired source entry with embedding $\mathbf{e}^s$ and a candidate destination with embedding $\mathbf{e}^d$, the unnormalized score is:
$$
    s(e^s, e^d) = \mathrm{MLP}_{\text{sink}}([\mathbf{e}^s \,\|\, \mathbf{e}^d]),
$$
where $[\cdot \,\|\, \cdot]$ denotes concatenation and the MLP maps $\mathbb{R}^{2d} \to \mathbb{R}$. The sink distribution is obtained by applying softmax over all valid candidates. The candidate sets are kept local: for \textsc{flow}, the destination must share at least one atom with the source, restricting the candidate set to $O(|\mathcal{V}|)$ per fired source; for \textsc{add} from atom $a$, the candidates are $\{e_{a,a}\} \cup \{e^H_a\} \cup \{e_{a,j} \mid j \neq a\}$, yielding at most $|\mathcal{V}| + 1$ candidates. For virtual bond entries (destinations that do not yet exist in $\mathcal{G}_t$), representations are constructed on-the-fly from the atomic embeddings of the endpoint atoms.

\subsection{Training details}
\label{app:training_details}

The full MAELLE model has approximately \textbf{16M} trainable parameters with hidden dimension $d = 256$ throughout. The GNN backbone uses \textbf{2} GAT layers; the encoder applies \textbf{4} fully-connected attention layers (\textbf{2} atom-level + \textbf{2} entry-level), and the decoder applies \textbf{6} fully-connected attention layers (\textbf{3} atom-level + \textbf{3} entry-level), with the latter two stages including cross-attention to the encoder. We train for $\sim$\textbf{50} epochs (about \textbf{72} hours on a single \textbf{NVIDIA H100} GPU) with an effective batch size of \textbf{64}, using gradient accumulation to fit memory.

\subsection{LLM-as-a-judge for benchmarking side product prediction}
\label{app:llm_judge}

To evaluate the chemical plausibility of predicted products beyond exact-match top-$k$ accuracy (Section~\ref{sec:side_products}), we use a large language model as a chemistry judge. Concretely, we query \texttt{gemini-3-flash} with \emph{minimal} thinking effort. 
For each reaction, the LLM is given the reactant SMILES and a \emph{anonymized} list of candidate products pooled across MAELLE, Molecular Transformer, and Graph2SMILES (with the ground-truth product additionally injected as a sanity-check item). The judge has no information about which model produced which candidate. Each of the three independent runs uses a fresh shuffle and a non-zero sampling temperature, providing the standard deviation reported in Figure~\ref{fig:side_prod}a.

\paragraph{Prompt structure.}
The query consists of a system prompt followed by a per-reaction user message assembled from three templated blocks (\textsc{prefix}, \textsc{inputs}, \textsc{suffix}):

\begin{quote}
\small
\textbf{System prompt.} \emph{You are an expert organic chemist. Your task is to evaluate whether predicted products are chemically plausible outcomes of a given reaction. A product is plausible if it could reasonably arise from the listed reactants under standard organic-chemistry conditions, considering valence, connectivity, common reactivity, and likely mechanisms. You do not need the product to be the major product --- minor or side products are acceptable as long as they are mechanistically reasonable.}
\medskip

\textbf{Prefix.} \emph{Reaction \texttt{<reaction\_id>}. Reactants (SMILES): \texttt{<reactant\_smiles>}.}
\medskip

\textbf{Inputs.} \emph{Below is a list of candidate products. Evaluate each candidate independently.}
\begin{enumerate}
    \item[] \texttt{[1] <product\_smiles\_1>}
    \item[] \texttt{[2] <product\_smiles\_2>}
    \item[] \texttt{\dots}
    \item[] \texttt{[n] <product\_smiles\_n>}
\end{enumerate}

\textbf{Suffix.} \emph{For each candidate, provide a brief mechanistic reasoning and a plausibility verdict (\texttt{true}/\texttt{false}). Return your answer as a JSON object of the form:}
\begin{verbatim}
{
  "verdicts": [
    {"index": 1, "reasoning": "...", "plausible": true},
    {"index": 2, "reasoning": "...", "plausible": false},
    ...
  ]
}
\end{verbatim}
\end{quote}

\paragraph{Aggregation.}
For each reaction and each model, the top-$k$ plausibility rate is the fraction of the model's first $k$ predictions that the judge marks \texttt{plausible}. We report the mean and standard deviation over three independent LLM runs (Figure~\ref{fig:side_prod}a). To control for systematic over- or under-permissiveness of the judge, we additionally inject the ground-truth product into the candidate list and report its plausibility coverage as a sanity check; across runs we obtain $81.1 \pm 0.1\%$ ground-truth coverage, indicating a moderately conservative but stable judge.

\paragraph{Anonymization and de-duplication.}
Because MAELLE typically yields fewer unique candidates than beam-search SMILES models, for every test case we cap the per-model candidate list at the smallest number of unique predictions across the three models, ensuring a balanced comparison. Within a single query the candidate list is shuffled, and SMILES are canonicalized with RDKit before being passed to the judge so that surface-form differences do not bias the verdict.

\section{Dataset}\label{app:data}

All experiments use the USPTO-480K reaction corpus~\citep{Lowe2012, Lowe2017}.
This appendix describes the preprocessing applied to the corpus
(Appendix~\ref{app:preproc}), the standard and re-split partitionings used in the
main text, and the construction of the out-of-distribution test axes
(Appendix~\ref{app:ood}).

\subsection{Preprocessing}
\label{app:preproc}
The following steps are applied once to every reaction in the corpus, prior to
any split.

\paragraph{Re-mapping.}
USPTO-480K ships with atom-mapped SMILES of the form
\texttt{reactants>reagents>products}, but the provided maps are noisy and
occasionally inconsistent. Because MAELLE's electron-site alignment relies on a
correct atom correspondence between reactants and products
(Section~\ref{sec:problem_statement}), we discard the original maps and re-map
every reaction with RXNMapper~\citep{schwaller2021rxnmapper}, yielding a
consistent atom mapping across the corpus.

\paragraph{Electron-move interpolation.}
For each re-mapped reaction we solve the optimal-transport problem of
Section~\ref{sec:ot} once, ahead of training, to obtain the move set $\mathcal{M}$
that transforms the reactant occupation $o_x$ into the product occupation $o_y$.
These precomputed moves are cached and reused across epochs; the per-step
Bernoulli sampling of the mixture path (Section~\ref{sec:dfm}) is then applied at
training time.

\paragraph{Atom maps and model inputs.}
The re-mapped correspondence is used to construct MAELLE's occupation targets and
is therefore retained for our model. For the SMILES- and graph-based baselines,
which do not consume atom maps, maps are stripped from the inputs so that all
baselines receive the same map-free reactions they were designed for. This is a
property of each model's input format, not an asymmetry in the data: every model
sees the same reactions, and only MAELLE makes use of the atom correspondence.

\paragraph{Deduplication.}
Exact duplicates are identified by a canonical reaction key
$\text{canon}(\text{reactant}) \gg \text{canon\_main}(\text{product})$, where
$\text{canon\_main}$ selects the largest product fragment by heavy-atom count
after stripping atom maps and canonicalizing with RDKit~\citep{rdkit}. Pairwise
overlap between all splits is verified to be zero on this key.

\subsection{Standard and low-MW splits}
\label{app:base-split}
For the standard benchmark (Section~\ref{sec:uspto_480K}) we use the corpus's
original train/validation/test partition. For the out-of-distribution study
(Section~\ref{sec:ood}) we re-partition the corpus as follows. The base
train/validation/IID-test partition is constructed by Gaussian sampling over the
product molecular weight (MW), with mean \num{250}\,Da and variance
\num{1643.782}\,Da$^2$ (std $\approx 40.5$\,Da); we refer to this as the
\emph{low-MW split}. The deliberately low mean ensures the training distribution
is well separated from the high-mass OOD axis below. The pooled train+validation
set is partitioned 80/20, giving \num{304228} training reactions after
deduplication, a validation set of \num{37582}, and an IID test set of
\num{37578} drawn from the remaining half of the validation fold.

\subsection{Out-of-distribution test axes}
\label{app:ood}



\paragraph{OOD Mass (high molecular weight).}
Similary to the training corpus, a test set was constructed using Gaussian sampling with mean \num{750} Da and variance of \num{6575.13} Da$^2$. The parameters of the Gaussian were chosen in order to ensure a clear separation between this set and the training-validation one. The test set described will be called hereon \emph{test\_ood\_mass} (\num{38909} reactions, median product MW 576\,Da). Reactions that were present in \emph{test\_ood\_mass} and also in the training-validation set were eliminated in the latter.
Because the Molecular Transformer must generate the full product SMILES autoregressively, the length of the output sequence scales with molecular size; the mechanism model, which predicts only local bond edits, is invariant to this factor. This split therefore directly probes
sensitivity to molecular size.

\paragraph{OOD Ester (non-methyl ester chemistry).}
An additional axis was constructed to isolate ester-forming reactions that
produce linkages other than methyl esters -- a transformation class
under-represented in the training set relative to its synthetic
importance.
Reactions belonging to the following named reaction classes were excluded
from all training and IID evaluation splits and instead pooled into a
dedicated test set, \emph{test\_ood\_ester} (\num{1779} reactions):
\textit{Ethyl esterification}, \textit{CO\textsubscript{2}H-\textit{t}Bu
protection}, \textit{O-Acetylation}, \textit{O-Piv protection},
\textit{O-Formylation}, \textit{Acetoxy thioether synthesis},
\textit{Salol reaction}, \textit{Yamaguchi esterification},
\textit{Yamaguchi lactonization}, \textit{Shiina macrolactonization},
\textit{2-Benzofuranone synthesis}, and \textit{Baeyer--Villiger
oxidation}.
In addition, eight reaction classes that produce ester linkages of
ambiguous type (\textit{Esterification}, \textit{Ester
Schotten--Baumann}, \textit{Fischer--Speier esterification},
\textit{Steglich esterification}, \textit{Mitsunobu ester synthesis},
\textit{Transesterification}, \textit{Carboxylic anhydride alcoholysis},
and \textit{Diazoalkane esterification}) were dropped from all splits
entirely to prevent soft leakage of ester-chemistry signal into the
training distribution.

\subsection{Summary statistics}

\begin{table}[h]
\centering
\caption{Split sizes, sampling MW (mean $\pm$ std of the Gaussian over product molecular weight), and OOD criteria for \texttt{uspto\_unified\_v3}.}
\label{tab:splits}
\begin{tabular}{lrcl}
\toprule
Split & $N$ & MW (mean $\pm$ std) [Da] & OOD criterion \\
\midrule
train                  & \num{304228} & $311.9\pm 90.6$  & -- \\
val                    & \num{37582}  & $312.7 \pm 90.6$  & -- \\
test\_iid              & \num{37578}  & $311.3\pm 90.5$  & none (IID) \\
test\_ood\_mass        & \num{38909}  & $585.2 \pm 83.0$  & product MW $> 529$\,Da \\
test\_ood\_ester       & \num{1779}   & $321.1 \pm 112.1$               & non-methyl ester chemistry \\
\bottomrule
\end{tabular}
\end{table}


\newpage

\end{document}